\documentclass{article}

\usepackage{arxiv}

\usepackage[utf8]{inputenc} % allow utf-8 input
\usepackage[T1]{fontenc}    % use 8-bit T1 fonts
\usepackage{hyperref}       % hyperlinks
\usepackage{url}            % simple URL typesetting
\usepackage{booktabs}       % professional-quality tables
\usepackage{amsfonts}       % blackboard math symbols
\usepackage{nicefrac}       % compact symbols for 1/2, etc.
\usepackage{microtype}      % microtypography
\usepackage{lipsum}		% Can be removed after putting your text content
\usepackage{amsmath,amssymb,amsfonts}
\usepackage{booktabs,graphicx,textcomp}
\usepackage{subfig,float,array,supertabular}
\usepackage{multirow,array}
\usepackage{hyperref,cite,url}
\usepackage{subfig,float}
\usepackage{algorithm,algorithmic}
\usepackage{enumitem}
\usepackage{tikz}
\usetikzlibrary{bayesnet}
\usetikzlibrary{arrows}
\usetikzlibrary{fit,positioning,arrows,automata}
\usepackage{bbm}
\usepackage{appendix}
\usepackage{comment}

\usepackage{enumitem}
\newtheorem{theorem}{Theorem}

\newtheorem{prop}{Proposition}
\newenvironment{proof}{{Proof:}}{\hfill$\blacksquare$}

\newcommand{\bx}{\mathbf{x}}
\newcommand{\bX}{\mathbf{X}}
\newcommand{\bV}{\mathbf{V}}
\newcommand{\bz}{\mathbf{z}}
\newcommand{\bZ}{\mathbf{Z}}

\newcommand{\bth}{\boldsymbol{\theta}}
\newcommand{\bmu}{\boldsymbol{\mu}}

\newcommand{\bze}{\boldsymbol{\zeta}}

\newcommand{\bths}{\boldsymbol{\theta}^{\star}}

\newcommand{\eqq}{$\,=\,$}

\newcommand{\bll}{{\tiny $\bullet$} }

\newcommand{\ben}{\begin{eqnarray*}}
\newcommand{\een}{\end{eqnarray*}}

\newcommand{\be}{\begin{eqnarray}}
\newcommand{\ee}{\end{eqnarray}}
\newcommand{\blll}{\hspace*{-3.5mm}{\tiny $\bullet$} }

\newcommand{\cca}{ {\cal C}}

\newcommand{\ts}{\text{s}}

\newtheorem{lemma}{Lemma}
\newtheorem{pos}{Postulate}

\def\hspp{\hspace*{.1in}}

\title{Inference for Multiple Object Tracking: A Bayesian Nonparametric Approach}
\author{
  Bahman ~Moraffah\\
  Department of Electrical, computer, and Energy Engineering\\
  Arizona State University\\
  Tempe, AZ, 85281 \\
  \texttt{bahman.moraffah@asu.edu} \\
}

\begin{document}
\maketitle

\begin{abstract}
In recent years, multi object tracking (MOT) problem has drawn attention to it and has been studied in various research areas. However, some of the challenging problems including time dependent cardinality, unordered measurement set, and object labeling remain unclear. In this paper, we propose robust nonparametric methods to model the state prior for MOT problem. These models are shown to be more flexible and robust compared to existing methods. In particular, the overall approach estimates time dependent object cardinality, provides object labeling, and identifies object associated measurements. Moreover, our proposed framework dynamically contends with the birth/death and survival of the objects through dependent nonparametric processes. We present Inference algorithms that demonstrate the utility of the dependent nonparametric models for tracking. We employ Monte Carlo sampling methods to demonstrate the proposed algorithms efficiently learn the trajectory of objects from noisy measurements. The computational results display the performance of the proposed algorithms and comparison not only  between one another, but also between proposed algorithms and labeled multi Bernoulli tracker.
\end{abstract}

% keywords can be removed
\keywords{ Bayesian nonparametric models  \and Multi object tracking\and Dependent Dirichlet process \and Dependent two-parameter Poisson-Dirichlet process \and Markov chain Monte Carlo}
%\twocolumn
\section{Introduction}

During the last decade, multiple object tracking (MOT) is a challenging and computationally intensive problem that has appeared in several different contexts and applications, including computer vision \cite{Com03, Koch2016, cox1996}, driver assistance \cite{Avidan2004, Nieto2014, Vo2015}, surveillance \cite{Kettnaker1999}, and radar target tracking \cite{vo2017, Aok16}. The MOT problem entails the estimation of time dependent and unknown number of the objects based on incoming data at each time step. Incoming data may be highly noisy and/or clutter. The estimation algorithms are reasonably robust to the noise model , however, the estimations are highly biased in the specification of clutter models \cite{bar1995, vo1999}. Mechanisms with dependency constraint are flexible and easy to control \cite{Ahmed2008, Bartlett2010, caron2012, caron2017}. 

There has been various approaches to the MOT problem. In \cite{bar1990, mullane2011, reuter2014, vo2017,vo2014labled, Wan17}, this problem is discussed through random finite set (RFS) methods. with probability hypothesis density filtering and multi-Bernoulli filtering, are used to model and track object states. Most methods pair objects to their associated estimated state parameters using clustering methods after tracking \cite{Vo2009}. In the recent studies, \cite{vo2014labled, Vo2013, Vo2014} the labeled multi-Bernoulli filtering method uses labeled RFS to estimate the objects identity, though at a high computational cost and high signal to noise ratio. In \cite{Aok16}, maximum a posteriori probability estimates of the object labeling uncertainties are integrated with a multiple hypothesis tracking algorithm.
However, in the recent studies, nonparametric approaches to MOT have drawn attention \cite{Bahman_rept}. To describe dependency among a collection of stochastic processes, the dependent Dirichlet process (DDP) is introduced \cite{Mac99,Mac2000}. In \cite{fox2007, fox2011}, a hierarchical Dirichlet process on the modes is employed to provide a prior over the unknown number of unobserved modes when tracking with maneuvering. A hierarchical model is also introduced to model the dependent measurement to track an object \cite{moraffahfusion2019_hier}. We  introduced a dependent Dirichlet process modeling for MOT \cite{moraffah2018}.  In \cite{campbell2013, neiswanger2014, topkaya2013}, a dependent Dirichlet process (DDP) mixture model is used to develop a clustering algorithm for batch-sequential data with time-varying clusters. The dependency with respect to covariates was introduced in \cite{arbel2016, griffin2011, Mac2000, mac2000unpublished, Mac99}. MOT is also discussed in terms of random infinite trees and diffusion processes in a nonparametric fashion \cite{moraffah2019}. In this paper, we introduce a family of density estimators for time dependent tracking algorithms constructed by Pitman-Yor processes. A Markov chain Monte Carlo (MCMC) inferential method integrates the distributions to update the time dependent states. Two distribution Poisson-Dirichlet process, Pitman-Yor process, was first introduced by Pitman and Yor in \cite{ pitman2002, pitman1997} and it was then used as prior in different research areas \cite{Bartlett2010, blei2011, caron2017}. Time varying P{\'o}lya urn approach for time varying Dirichlet process mixture and Pitman-Yor processes were proposed as prior on parameters over the observations, however they do not capture the full dependency or are not marginally a Dirichlet process or a Pitman-Yor process \cite{Ahmed2008, moraffahetal2019, blei2011, caron2012, caron2017, moraffahfusion2019}. The focus of this paper is to study the multi-object tracking problem using nonparametric approaches such as the dependent Dirichlet process and dependent Pitman-Yor process as a prior over the objects state distributions. We propose a class of algorithms and discuss the statistical properties of the models. Our main algorithms establish that our model outperforms existing methods and is computationally inexpensive. We also show that the introduced methods (A) are marginally well defined and hence there is an efficient way to do inference, (B) are consistent under mild conditions  (C) achieve the minimax rate and in this sense is optimal.

\subsection{Contributions and Organization}
Our main contribution is to construct novel nonparametric methods for MOT problem and describe their statistical properties. We define time varying models based on the dependent Dirichlet process and the two-parameter Poisson-Dirichlet process on the state of the objects that have more flexibility compared to existing methods such as  RFS models. Our proposed model is an improvement in tracking, time efficiency, and implementation. Our models (1) capture the full time dependency among the states and its parameters based on a dependent Dirichlet process and/or a two parameter Poisson-Dirichlet process such that the marginal distribution follows a Dirichlet process/Pitman-Yor process, which makes the inference efficient, consistent, and robust, (2) converges at the optimal frequentist rate (minimax rate) (3) capture both birth and death process in MOT and simply labels each objects and accurately provides the number of the objects as well as the object trajectory at each time step, and (4) are simple to design a MCMC model that can accurately estimates the trajectory of the objects and outperforms the existing approaches such as RFS. 

The rest of the paper is organized as follows. Section \ref{back} provides a review of the standard Dirichlet and Pitman-Yor processes. Section \ref{PF} presents the multiple object tracking problem with time-dependent cardinality. In section \ref{DDP_Model}, we construct a novel family of time-varying models based on a dependent Dirichlet process (DDP) as prior; and we develop a Markov chain Monte Carlo (MCMC) inference method in section \ref{ddplearning}. We then prove the convergence of the algorithm and discuss some properties of our proposed method in section \ref{prop}. We then introduce an extension of the proposed nonparametric method through a time-dependent Pitman-Yor prior, section \ref{DPY}, and develop the corresponding MCMC learning method in section \ref{dpylearning}. The properties of this Pitman-Yor based model is discussed in section \ref{pyprop}. We, through simulations, demonstrate the performance of our proposed methods and compare them to one another and also the RFS based method labeled multi-Bernoulli filter in Section \ref{sim}.

\section{Background}
\label{back}
In recent years, the ubiquitous influence of Bayesian nonparametric models in modeling and density estimation to avoid the restrictions of parametric methods is well establsihed. In particular, the family of infinite dimensional space of random measures such as Dirihclet process \cite{ferg1973} and Two-Parameter Poisson-Dirichlet Process (Pitman-Yor Process) \cite{pitman1997} as priors have become very popular in statistics and machine learning. Dirichlet process mixture models \cite{antoniak1974} and Pitman-Yor Mixture models \cite{pitman2002, pitman1997} have played an important role as substitutes for finite mixture models to estimate the density and perform clustering. These methods, if designed appropriately, can be used to easily do the inference. In the following section we briefly describe two nonparametric models, which we will employ throughout this paper.

\subsection{Dirichlet Process}
Dirichlet process(DP) is a class of nonparametric models that defines a prior on the space of probability distributions on the infinite dimension parameter space $\Theta$ \cite{ferg1973, teh2011}. A DP with a concentration parameter $\alpha$ and base distribution $H$ on the parameter space $\Theta$ is denoted by $DP(\alpha, H)$ and is defined as 
\begin{equation}
\label{dp}
G(A) = \sum\limits_{j=1}^{\infty} \pi_j \delta_{\theta_j}(A), \hspace{0.5cm} \theta_j\sim H,\hspace{0.2cm} \text{and} \hspace{0.2cm} \pi_j\sim \text{GEM}(\alpha)
\end{equation}   
where $\delta_{\theta_j}(A) = 1, \text{if} \hspace{0.1cm}\theta_j\in A$ and $\delta_{\theta_j}(A) = 0, \text{if} \hspace{0.1cm}\theta_j\notin A$ and $\text{GEM}(\alpha)$ follows the stick breaking representation discussed in \cite{sethuraman1994}: 
\begin{flalign}
V_j &\sim \text{Beta}(1,\alpha)\hspace{2.5cm} j = 1,2,\dots\notag\\
\pi_j &= V_j\prod\limits_{i = 1}^{j-1}(1-V_i)\hspace{2cm} j = 1,2,\dots .
\end{flalign}
Note that $G(\cdot)$ is a probability random measure and is shown to be discrete with probability one. Assume that we receive a fixed number of observations $\mathcal{Z} = \{\bz_1, \dots, \bz_m\}$ where $\bz_j$'s given parameters are independently and  identically drawn from $F$, where 
\begin{equation}
F(\cdot) = \int_{\theta}f(\cdot|\theta)dG(\theta)
\end{equation}
where $f(\cdot|\theta)$ is the density and $G(\theta)$ is the mixing distribution drawn according to a DP.  From \ref{dp}, one can define the infinite mixture model as follows:
\begin{flalign}
\label{mixturedp}
G|\alpha, H &\sim DP(\alpha, H)\notag\\
\theta_j | G &\sim G\\
z_j|\theta_j &\sim f(\cdot|\theta_j).\notag
\end{flalign}
Equation \ref{mixturedp} is known as Dirichlet process mixture model (DPM model). It can be shown that the expected number of clusters using DP model is $\alpha \log m$, where $m$ is the number of data. 
\subsection{Two-Parameter Poisson-Dirichlet Process}
The class of two-parameter poisson-Dirichlet processes (Pitman-Yor processes) is a wide class of distributions on random probability measure that contains Dirichlet processes. We denote the Pitman-Yor process $\mathcal{PY}(d,\alpha, H)$, where $H$ is base probability distribution. The parameters $0 \leq d < 1$ and $\alpha > -d$ are discount and concentration parameters, respectively. The case where $d = 0$ agrees with a $DP(\alpha, H)$. The Pitman-Yor process is a subclass of $d$-Gibbs Processes which shares the essential properties of Dirichlet processes. Pitman-Yor processes are most suited for data with the power-law property \cite{teh2006}. A realization of the $\mathcal{PY}(d,\alpha,H)$ is a discrete random measure that can be constructed using stick breaking as follows:
\begin{equation}
\label{PY}
G(A) = \sum\limits_{j=1}^{\infty} \pi^*_j \delta_{\theta_j}(A) \hspace{0.5cm} \theta_j\sim H
\end{equation}
and $\pi^*_j$ is the size-biased order of $\pi_j$ where $\pi_j = V_j\prod\limits_{i = 1}^{j-1}(1-V_i)$ and $V_j \sim \text{Beta}(1-d, \alpha+jd)$. Pitman-Yor process defines a prior on the probability distribution over the infinite dimension space of parameters. Assume that $\mathcal{Z}$ is the set of measurements drawn from distribution $F$, the Pitman-Yor mixture model for $j = 1, \dots, m$ is given by
\begin{flalign}
\label{mixturepy}
G|\alpha, H &\sim \mathcal{PY}(d, \alpha, H)\notag\\
\theta_j | G &\sim G\\
z_j|\theta_j &\sim f(\cdot|\theta_j).\notag
\end{flalign}

With the Pitman-Yor process, it can be shown that the expected number of clusters is $\alpha m^d$. Following the power- law, the higher the number of unique (non-empty) clusters, the higher the probability of having even more unique clusters \cite{teh2006, Sato2010}.

It is worth mentioning that, although these models are well suited for many problems, there are many situations that time varying distributions are required to capture the dependency. In multi object tracking problem, one thus, needs to design time-evolving distributions such that the data-driven posterior inference problem is efficient and easy to compute. The previously proposed nonparametric methods do not capture the full dependency or the marginal distribution is not preserved. We introduce a novel family of first order time-dependent Dirichlet and a time-dependent Pitman-Yor prior process that capture the full dependency such that the marginal distribution at each time step given the configurations at previous time step follows a Dirichlet and Pitman-Yor process, respectively. This property for the introduced generative model not only proposes an efficient way to compute but also matches the minimax rate. We detail our approach for a MOT problem next.

\section{Problem Formulation}
\label{PF}
The goal of any multi object tracking model is to (A) successfully estimate the trajectory of each object given the observation data and (B) find the number of the objects at each time step. Given the state vector configurations at previous time step and current time observations, we propose two nonparametric algorithms to satisfy (A), (B). 

We consider the problem of multi object tracking with time varying number of objects remaining, entering, and/or leaving the field of view (FOV). Assume the time-dependent object and measurement cardinality $N_k$ and $M_k$ at time step $k$, respectively. Suppose that object state vectors $\bX_k = \{\bx_{1,k}, \dots, \bx_{N_k,k}\}$ taking values in state space $\mathcal{X}$ and observation vectors $\bZ_k = \{\bz_{1,k}, \dots, \bz_{M_k, k}\}$  taking values in observation space $\mathcal{Z}$, at time $k$. Assume space $\mathcal{X}$ and $\mathcal{Z}$ are Polish spaces. Assume that $N_{k}$ and $M_{k}$ are the unknown cardinality of the object states and observations at time $k$, respectively. Given the state vector at time $(k-1)$, three possible situations may occur: 

\begin{enumerate}
\item[(a)] {\bf{Survival and Transition}}: the object remains in the FOV with probability $\text{P}_{k|k-1}$ and its state transitions to the next time step $k$ according to the transition kernel $\mathbb{Q}_{\underline{\theta}}(\bf{x_\ell(k-1)}, \cdot)$ with unknown parameters ${\underline{\theta}}$. 
\item[(b)] {\bf{Death}}: the object leaves the FOV with probability with probability $1-\text{P}_{k|k-1}$.
\item[(c)]{\bf{Birth}}: new object enters the scene. 
\end{enumerate}

Throughout this paper, we assume each measurement is generated by only one object and the measurements are independent of one another. An object with state vector $\bx_k \in \bX_k$  generates an observations $\bz_k \in \bZ_k$ with likelihood distribution $p(\bz_k| \bx_k)$. The nonparametric models are a versatile tool to model a prior, however it cannot capture evolution over a period of time. Therefore, we need a more powerful tool to capture (a)-(c) over time. To model a collection of random distributions that are related but not identical, we define dependent nonparametric models to not only satisfies (a)-(c) but also captures time dependency. In what follows, we introduce two class of time-dependent nonparametric multi object- state prior models that given the process at time $(k-1)$ satisfy the following at time $k$:

\begin{enumerate}[label=(\roman*)]
\item {\bf{Survival}}: Given the $\ell$th state at time $k-1$, $\bx_{\ell,k-1}$, define $\text{P}_{\ell, k \mid k-1}:\Omega\rightarrow[0,1]$ to be the survival probability of state $\ell$ at time $k-1$.\label{1st}
\item{ \bf{Transition}}: Let $\nu:\Omega\times\mathcal{B}\rightarrow\mathbb{R}^{+}$ be  the transition kernel. For each survived cluster, the cluster parameters are evolved through 
$\theta_{\ell,k}\sim\nu(\theta^*_{\ell, k-1}, \cdot)$. \label{2nd}
\item {\bf{Trajectory}}: Given the measurements, update the marginal (predictive) distribution.  \label{3rd}
\end{enumerate}
Employing \ref{1st} - \ref{3rd} provides nonparametric frameworks such that an object may perhaps disappear or remain and evolve over time. The evolution of the object throughout the time is recorded and is updated based on observing the measurements and forms the trajectory.  We introduce a time-dependent two parameter Poisson-Dirichlet and a time-dependent Dirichlet processes to capture dependency among the object states such that the marginal distributions follow a Pitman-Yor and Dirichlet process, respectively. The graphical model capturing these frameworks presented in Fig \ref{alg_gm}. 
%\section{Graphical Model Representing DPY}
\begin{figure}[t]
\centering
\tikz \node [scale=0.9, inner sep=0] {
\begin{tikzpicture}
align = flush center,
\tikzstyle{plate caption} = [caption, node distance=0pt, 
inner sep=0pt, below left=3pt and 0pt of #1.south east]

\node[latent, shape=circle,draw, inner sep=1.5pt, minimum size=1em] (thetak-1) { $\theta_{\ell,k-1}$};%
%\tikz\draw [thick,->]  (thetak-1.west) to (-1,2);
\draw [->] (-3,0) -- (thetak-1);
\node[latent, right = 2.5cm of thetak-1, shape=circle,draw, inner sep=-1pt, minimum size=1em] (theta_tran) {$\theta_{\ell,k|k-1}$};
\node[latent, right = 5.5 cm of thetak-1, shape=circle,draw, inner sep=4.5pt, minimum size=2.2em] (theta_knew) {$\theta_{\ell,k}$};
\node[latent, right = 3cm of theta_knew, shape=circle,draw, inner sep=2.0pt, minimum size=2.2em, yshift = 0.4 cm ] (theta_k) {$\theta_{\ell,k+1}$};
\edge{thetak-1}{theta_tran}
\draw [->,black] (theta_tran.north east) to [out=28, in=160] (theta_k.west);
\draw [->] (theta_k.east) -- (12.99,0.4);
%\edge{theta_tran}{theta_k}
\edge{theta_knew}{theta_k}
 \plate[inner sep=.25cm,xshift = -0.1cm, yshift=.25cm] {plate1} {(thetak-1)} {$\ell = 1,\dots, N_{k-1}$}; 
 \plate[inner sep=.25cm,xshift = -0.1cm, yshift=.25cm] {plate2} {(theta_tran)} {$\ell = 1,\dots, D_{k|k-1}$}; 
  \plate[inner sep=.25cm,xshift = -0.1cm, yshift=0cm] {plate3} {(theta_knew)} {$\infty$}; 
  \plate[inner sep=.25cm,xshift = -0.1cm, yshift=0cm] {plate4} {(plate2) (plate3)} {}; 
   \plate[inner sep=.25cm,xshift = -0.1cm, yshift=0cm] {plate5} {(theta_k)} {$\ell = 1,\dots, D_{k+1|k}$}; 
\node[latent, below =1.5cm of thetak-1, xshift = 1.5cm, shape=circle,draw, inner sep=1.5pt, minimum size=1em] (sk-1) { $\bX_{k-1}$};%
\node[latent, right = of sk-1, xshift = 2.5cm, shape=circle,draw, inner sep=-4pt, minimum size=3.1em] (sk) { $\bX_{k}$};%
\node[obs, below = 1.5 cm of sk-1, shape=circle,draw, inner sep=1.5pt, minimum size=3.05em] (zk-1) { $\bZ_{k-1}$};%
\node[obs, right = of zk-1, xshift = 2.5cm, shape=circle,draw, inner sep=-4pt, minimum size=3.1em] (zk) { $\bZ_{k}$};%
\edge{sk-1}{sk}
\edge{sk-1}{zk-1}
\edge{sk}{zk}
\draw [->] (-1.2,-2.6) -- (sk-1);
\draw [dotted, ultra thick] (-2.5,-2.6) -- (-2.2,-2.6);
\draw [->] (sk.east) -- (8.6,-2.6);
\draw [dotted, ultra thick] (9.7,-2.6) -- (10,-2.6);
\edge{thetak-1}{sk-1}
%\edge{thetak-1}{zk-1}
\draw [->,black] (thetak-1.south)  to [out=-100, in=-210] (zk-1.north west);
\edge{theta_knew}{sk}
\draw [->,black] (theta_knew.south east) to [out=-80, in=30] (zk.north east);

\edge{theta_tran}{sk}
\draw [->,black] (theta_tran.south)  to [out=-100, in=-210] (zk.north west);

    \end{tikzpicture}};
    \caption{Graphical model capturing the temporal dependence.}
    \label{alg_gm}
\end{figure}
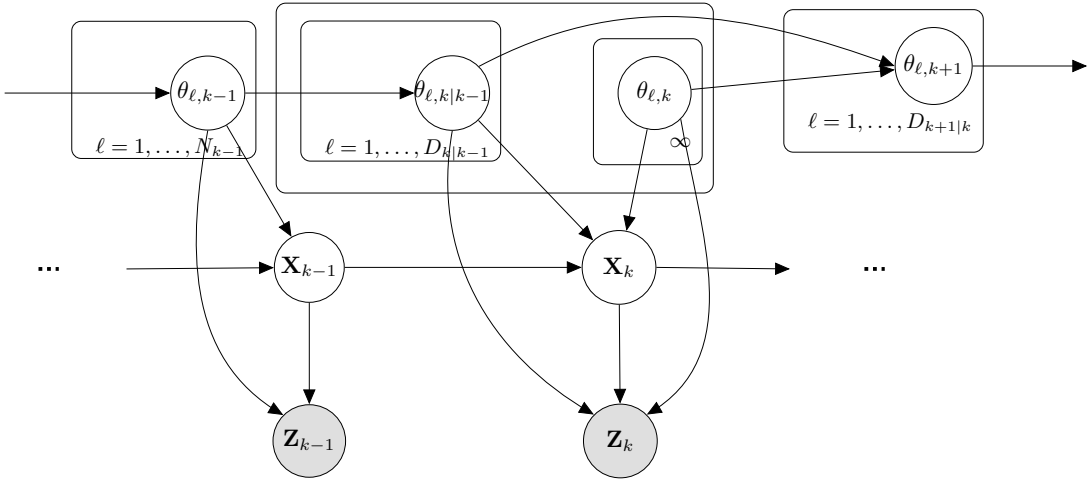

\section{Nonparametric MMT: Dependent Dirichlet Process Construction}
\label{DDP_Model}
\subsection{Evolutionary Time Varying Model Construction}

In this section, we propose an evolutionary time dependent model to multiple object tracking based on our proposed dependent Dirichlet process (DDP) to infer the object trajectory and labels. The proposed DDP evolutionary Markov modeling(DDP-EMM) approach, can be used to learn multiple object clusters or labels over related information. The DDP-EMM algorithm is different from random finite set (RFS) based algorithms for characterizing multiple object states and measurements \cite{Vo2015, vo2014labled}. In particular, our approach directly incorporates learning multiple parameters through related information, including object labeling at the previous time step or labeling of previously considered objects at the same time step. In particular, the choice of the DDP as a prior on the object state distributions is based on the following dynamic dependencies in the state transition formulation: (I) the number of objects present at time step $k$ not only depends on the number of objects that were present at the previous time step $(k-1)$ but it also depends on the popularity of the object (preferential attachment), (II) the clustering index of the parameter state of the $\ell$th object at time step $k$ depends on the clustering index of the parameter states of the previous $(\ell-1)$ objects at the same time step $k$, and (III) model a new object entering the scene without requiring any prior knowledge on the expected number of objects. Note that we assume that this process is de Finetti exchangeable meaning the exchangeable partition probability function (EPPF) depends only on the size of the clusters. We may thus assume that the $\ell$th object is the last one to consider for clustering. The DDP-EMM algorithm is discussed next in detail and summarized in Algorithm 1. In particular, we provide: (i) the information available at time step $(k -1)$, (ii) how this information transitions from time step $(k - 1)$ to time step $k$, and  (iii) how the state transition stochastic model is constructed at time step k to form the multiple object state prior.

{\bf{(i) Available Parameters at time $(k-1)$ :}} The DDP-EMM algorithm assumes the following parameters available in time step $(k-1)$:
\begin{itemize}
\item $\bx_{\ell,k-1}$,  $\ell$th object state parameter vector, $\ell \eqq 1,2, \ldots, N_{k-1}$
\item $\bth_{\ell, k-1}$,  $\ell$th object-state DP cluster parameter vector
\item $\Theta_{k-1} \eqq \{\bth_{1,k-1},\ldots,  \bth_{N_{k-1},k-1}\}$, collection of the cluster parameters
\item $D_{k-1}$ = {\footnotesize $\#$}  of unique DP clusters used as state prior
\item $\Theta^\star_{k-1} \eqq \{\bths_{1,k-1},   \ldots , \bths_{D_{k-1},k-1} \}$, collection of the unique parameters such that $\Theta^\star_{k-1}  \subseteq \Theta_{k-1}$
\item $V^\star_{k-1}$ = vector of size $D_{k-1}$ where $\left[V^\star_{k-1}\right]_i$ is the number of objects in the $i$th cluster $i \eqq 1, \ldots,D_{k-1}$.
%\item  $n_{i , k-1}$ = {\footnotesize $\#$} of objects in the $i$th cluster, $i \eqq 1, \ldots,D_{k-1}$.
\end{itemize}
The induced {\em cluster assignment indicator sequence} at time $k-1$ is defined as 
\begin{equation}
\mathcal{C}_{k-1} =  \{ c_{1,k-1}, \ldots, c_{D_{k-1},k-1} \}, 
\label{cl_seq}
\end{equation}
where $c_{j} \in \{1,\dots, D_{k-1}\}$. Let $\mathcal{CA}_{k-1}$ be the collection of clustering assignment up to time $(k-1)$, i.e., $\mathcal{CA}_{k-1} = \{\mathcal{C}_{1}, \dots, \mathcal{C}_{k-1}\} $.

{\bf{(ii) Algorithm parameters transitioning from time $(k-1)$ to time $k$:}}
It is assumed by problem statement that if $\bx_{\ell,k-1} \in \bX_{k-1}$, the object with the state $\bx_{\ell, k-1}$ can disappear from the FOV with probability $1 - \text{P}_{k|k-1}$ or can stay in the scene with probability $\text{P}_{k|k-1}$ and transition to a new state with the transition kernel $\mathbb{Q}_{\underline{\theta}}(\bf{x_\ell(k-1)}, \cdot)$. Let $\Theta^\star_{k|k-1}$ be the set of unique transitioned parameters to time step k. We assume if all the objects in a cluster leave the scene the cluster no longer exist. The Bernoulli process associated with appearance/disappearance of the objects during transition from time $(k-1)$ to time $k$ is defined as:
\begin{equation}
\mathcal{B}_{k-1} = \{ \ts_{1, k \mid k-1},  \dots, \ts_{N_{k-1}, k \mid k-1} \}
\end{equation}
where $\ts_{\ell , k \mid k-1} \sim Bernoulli(\text{P}_{\ell,k|k-1})$. Note that $\ts_{\ell , k \mid k-1} = 1$ indicates the survival of the $j$th object and transitioning to time $k$. We assume if all the objects in a cluster leave the scene the cluster no longer exist. Define the vector $V^\star_{k|k-1}$ to be the vector of size $D_{k-1}$ with entries indicating the size of each cluster after transitioning to time $k$. Note that some of elements may be zero. Since a cluster of size zero suggests that the cluster no longer exists, we may eliminate zeros in $V^\star_{k|k-1}$. We thus define the {\em cluster survival indicator} corresponding to nonempty clusters as 
\begin{equation}
\mathcal{CS}_{k|k-1} = \{\lambda_{1,k|k-1}, \dots, \lambda_{D_{k-1},k|k-1}\}
\end{equation}
where $\lambda_{j,k|k-1} \in \{0,1\}$. Note that $\left[V^\star_{k|k-1}\right]_j = 0$ implies $\lambda_{j,k|k-1} = 0$ and if there is at least one object in the $j$th cluster, then $\lambda_{\ell,k|k-1} = 1$. Note that the number of non-zero clusters that transitions to time $k$ is $D_{k|k-1} = \sum_{j} \lambda_{j,k|k-1}$. 

{\bf{(iii) DDP Prior Construction at time $k$:}} The DDP-EMM algorithm employs the parameters from time $(k-1)$ and the transition step to estimate the state distribution. Each cluster with $\lambda_{j,k|k-1} = 1$, $j \leq D_{k|k-1}$, a non-zero cluster, transitions to time $k$ according to the transition kernel $\nu(\theta^\star_{j,k-1},\cdot)$. Assume $\theta_{j,k}$ is the $j$th cluster parameter, we construct a dependent Dirichlet process as follow:
\begin{description}
\item[Case 1:] The $\ell$th object is assigned to one of the survived and transitioned clusters from time $(k-1)$ which is occupied by at least one of the previous $\ell -1$ previous objects. The survival of each object is determined by the survival indicator $s_{\cdot,k|k-1} \in \mathcal{B}_{k-1}$. We assume de Finetti exchangeability and thus we may assume the $\ell$th object is the last one to cluster.  The object selects one of these clusters with probability:
\begin{flalign}
 \label{case1}
\Pi^1_{j,k} (\textit{Choosing jth cluster} | {\bf{\theta}_{1,k}},\dots, {\bf{\theta}_{\ell-1, k}})=  \frac{\left[V_k\right]_j+  \sum \limits_{i=1}^{D_{k-1}} \left[V^\star_{k|k-1}\right]_i \lambda_{i, k \mid k-1}   \delta_i(c_{j, k})}{g_{\ell-1, k-1}} 
\end{flalign}
 where $|{\cal A}|$ is the cardinality of set ${\cal A}$ and 
  $\delta_i(\cdot)$ is the Dirac delta function, 
defined as  $\delta_i({\cal A}) \eqq 1$ if  $i $$\in$$ {\cal A}$ and 
$\delta_i({\cal A}) \eqq 0$ if  $i$$\notin$$ {\cal A}$.
The normalization term in \ref{case1} 
is given by 
$$g_{\ell-1, k-1} = (\ell-1)+ \sum\limits_{j}^{\ell-1} \sum_{i=1}^{D_{k-1}} 
\left[V^\star_{k|k-1}\right]_i \lambda_{i, k|k-1}  \delta_i(c_{j, k}) + \alpha$$, where $\alpha$$>$$0$ is the concentration parameter.

Assume the space of states, $\mathcal{X}$, is Polish, given equation \ref{case1} state distribution is drawn from as:
\small
\begin{equation}
\label{state1}
 p({\bf{x_{\ell,k}}} | {\bf{x_{1,k}}},\dots, {\bf{x_{{\ell-1,k}}}}, {\bf{X}}_{k|k-1}, \Theta^\star_{k|k-1}, \Theta_k) =  \mathbb{Q}_{\underline{\theta}}(\bf{x_{\ell, k-1}}, {\bf{x_{\ell, k}}})  f({\bf{x_{\ell, k}}} | {\bf{\theta^\star_{\ell, k}}})
\end{equation}
\normalsize
For some density $f$.
\item[Case2:] The $\ell$th object is assigned to one of the survived and transitioned clusters from time $(k-1)$. However, this cluster has not yet been assigned to any of the first $\ell-1$ objects. The object selects such a cluster with probability:
\small
\begin{flalign}
\label{case2}
\Pi^2_{j,k} (\textit{Choosing jth cluster that has not been selected yet} | {\bf{\theta}_{1,k}},\dots, {\bf{\theta}_{\ell-1, k}}) = \frac{\sum \limits_{i=1}^{D_{k-1}}\left[V^\star_{k|k-1}\right]_i \, \lambda_{i, k|k-1}\delta_i(c_{j, k})}{g_{\ell-1, k-1}}
\end{flalign}
\normalsize
%\small
%\begin{flalign}
%\label{eq2}
%\Pi_2 (\textit{Choosing jth cluster that has not been selected yet} | {\bf{\theta}_1(k)},\dots, {\bf{\theta}_{\ell-1}(k)}) = \frac{1}{g_{\ell-1, k-1}}  \sum_{
%{\substack{ i=1 \\ i \notin \mathcal{C}_{k} }}}^{D_{k| k-1} }
%\left[V^\star_{k|k-1}\right]_i \, \lambda_{i, k \mid k-1} 
%\end{flalign}
%\normalsize
where $g_{\ell-1, k-1}$ is defined the same as case 1. In case 2, $\bx_{\ell,k-1}$ and $\theta^\star_{\ell, l-1}$ transition to time $k$ using transition kernels  $\mathbb{Q}_{\underline{\theta}}(\bf{x_\ell(k-1)}, \cdot)$ and $\nu(\theta^\star_{\ell,k-1},\cdot)$, respectively. Assuming the state space $\mathcal{X}$ is Polish and given equation \ref{case2}, the state distribution is:
\small
\begin{flalign}
\label{state2}
 p({\bf{x_{\ell,k}}} | {\bf{x_{1,k}}},\dots, {\bf{x_{{\ell-1,k}}}}, {\bf{X}}_{k|k-1}, \Theta^\star_{k|k-1}, \Theta_k) = \mathbb{Q}_{\underline{\theta}}(\bf{x_{\ell, k-1}}, {\bf{x_{\ell, k}}})    \nu({\bf{\theta^\star_{\ell, k-1}}},{\bf{\theta_{\ell, k}}})  f({\bf{x_{\ell,k}}} | {\bf{\theta^\star_{\ell, k}}})
\end{flalign}
For some density $f$.
\item[Case3:] The object does not belong to any of the existing clusters; a new cluster parameter is drawn with probability:
\begin{equation}
\Pi^3_k (\textit{Creating new cluster} | {\bf{\theta}_{1,k}},\dots, {\bf{\theta}_{\ell-1, k}}) = \frac{ \alpha}{g_{\ell-1, k-1}}
\end{equation}
The state distribution thus may be drawn as:
\begin{equation}
\label{state3}
 p({\bf{x_{\ell,k}}} | {\bf{x_{1,k}}},\dots, {\bf{x_{{\ell-1,k}}}}, {\bf{X}}_{k|k-1}, \Theta^\star_{k|k-1}, \Theta_k) = \int_{\bf{\theta}}{f({\bf{x_{\ell,k}}} | \theta)} dH(\theta) 
\end{equation}
for some density $f$ and distribution $H$ on parameters. The algorithm \ref{DDP_alg1} summarizes this process. 
\end{description}

\begin{algorithm}[t] 
\begin{algorithmic}
\caption{DDP-EMM: Time-dependent arrival and survival process}
\label{DDP_alg1}

\STATE 
\hspace*{-.2in} {\textbf {At time}} $(k-1)$ \\
\blll $\bx_{\ell, k-1}$:  $\ell$th object state parameter vector, 
$\ell \eqq 1, \ldots, N_{k-1}$\\
 \blll $D_{k-1}$:   {\footnotesize $\#$} of unique DP  
 clusters used as  priors \\
\blll $V^\star_{k-1}$: vector of size $D_{k-1}$ where $\left[V^\star_{k-1}\right]_i$ is {\footnotesize $\#$} of objects in $i$th cluster\\
 \blll $\Theta^\star_{k-1} = \{ \bth^\star_{1, k-1},  \ldots , \bth^\star_{D_{k-1}, k-1}\}$: Cluster sequence of unique cluster parameters \\
 \blll $\mathcal{B}_{k-1}$ : Bernoulli collection of appearance and disappearance association\\
\blll $\mathcal{C}_{k-1}$ : cluster assignment  
\STATE
\hspace*{-.2in}  {\textbf {Transitioning from time $(k$$-$$1)$ to $k$}} \\
 \STATE \hspace*{-.18in} {\bf Input}:   $\text{P}_{\ell, k \mid k-1}$, 
transition kernel $\mathbb{Q}_{\bth_ {\ell, k}}(\bx_{\ell, k-1},  \bx_{\ell, k})$ \\
\STATE  \hspace*{-.18in} Draw $\ell$th state survival indicator
$\ts_{\ell , k \mid k-1}$$\sim$$\text{Ber} (\text{P}_{\ell, k\mid k-1})$\\
\STATE \hspace*{-.18in}  If $\ts_{\ell , k \mid k-1}\eqq 1$,  $\ell$th object
survives w.p. $\text{P}_{\ell, k \mid k-1}$ and transitions according to the \\
transition kernel 
{ \small{$\bx_{\ell, k} \sim \mathbb{Q}_{\bth_{\ell,k}}( \bx_{\ell, k-1},  \bx_{\ell, k} )$}}\\
 \STATE \hspace*{-.18in} Form the object survival indicator set: $\mathcal{CS}_{k|k-1} = \{ \ts_{1, k \mid k-1}, \ldots, \!
 \ts_{N_{k-1}, k \mid k-1}\! \}$ \\ 
 \blll Compute the {\footnotesize $\#$} of survived DP clusters after transitioning: $D_{k \mid k-1}$ \\
 \blll Form the size vector with entries $\left[V^\star_{k|k-1}\right]_j$, $j=1,\dots, D_{k|k-1}$
 %$V^\star_{k|k-1}$: vector of size $D_{k|k-1}$ where $\left[V^\star_{k|k-1}\right]_j$ : {\footnotesize $\#$} of objects in $j$th cluster \\ \vspace{-3mm}after transitioning\\
\STATE 
\hspace*{-.2in} {\textbf {At time } $k$}
 \STATE \hspace*{-.14in} {\textbf{Set}} $D_k \eqq D_{k \mid k-1}$
 \FOR{$ \ell=1$ \TO $D_k$}
\STATE Set $ \left[V_k\right]_\ell \eqq  \left[V^\star_{k|k-1}\right]_\ell$ \\
\IF{$\ell\leq D_{k} $ \AND $\ell$\text{th cluster already selected}}
\STATE \hspace{-.5mm}Draw $\bth_{\ell,k} \sim \nu(\bth_{\ell, k-1},\cdot)$ for cluster associated to $\ell$th object state w.p. $\Pi^1_{j,k}$ 
\STATE Draw $\bx_{\ell, k}|\bth_{\ell,k}$ for $\ell$th object state from \ref{state1}
\ELSIF{$\ell\leq D_{k} $ \AND $\ell$\text{th cluster not yet selected}}
\STATE \hspace{-.5mm}Draw $\bth_{\ell,k} \sim \nu(\bth_{\ell, k-1},\cdot)$ for cluster associated to $\ell$th object state w.p. $\Pi^2_{\ell,k}$ 
\STATE Draw $\bx_{\ell, k}|\bth_{\ell,k}$ for $\ell$th object state from \ref{state2}
\ELSE
\STATE \hspace{-.5mm}Draw $\bth_{\ell,k} \sim H$ for new cluster associated to$\ell$th object state w.p. $\Pi^3_{k}$ 
\STATE Draw $\bx_{\ell, k}|\bth_{\ell,k}$ for $\ell$th object state from \ref{state3}
%\STATE \hspace*{-.26in} Draw $\bth_{\ell,k}$ for cluster associated to $\ell$th object state with probability\\ $\Pi_i$, $i =1, 2, 3$ 
%from \ref{DDP-EMM}.\\
%\STATE \hspace*{-.26in} Draw $\bx_{\ell, k}$$\mid$$\bth_{\ell,k}$ for $\ell$th object state from predictive distributions in\\ \ref{state1}, \ref{state2}, \ref{state3}.
\ENDIF
 \ENDFOR 
\RETURN{ $\{ \bx_{1, k}, \bx_{2, k}, \ldots, \ldots \}$, 
$\{\bth_{1, k}, \bth_{2,k}, \ldots,  \ldots \}$ \\}
\end{algorithmic}
\end{algorithm}

This model (A) allows for modification of both cluster location and dependent weights, (B) ensures that the conditional distribution of DDP at time $k$ given the DDP at time $(k-1)$ is a Dirichlet process, (B) records the labels since it is defined in the space of partitions, (C) performs a standard MCMC method to do inference based on this  nonparametric model. We discuss (A)-(C) in the next theorems. 
\begin{theorem} 
\label{thm1}
Suppose that the space of state parameters is Polish. The dependent Dirichlet process in cases (1)-(3) define a Dirichlet process at each time step given the previous time configurations, i.e., 
\begin{equation}
\label{DDP-EMM}
\textit{DDP-EMM}_k | \textit{DDP-EMM}_{k-1} \sim DP\Big(\alpha, \sum_{\Theta_k} \Pi^1_{j,k} \delta_{\theta_{\ell,k}} + \sum_{\Theta^\star_{k|k-1}\setminus \Theta_{k}}\Pi^2_{j,k} \nu({\bf{\theta}^\star_{\ell, k-1}},{\bf{\theta_{\ell, k}}}) \delta_{{\bf{\theta_{\ell,k}}}} + \Pi^3_k H\Big).
\end{equation}
\end{theorem}
%{\textit{Sketch of proof:}} There are three scenarios: 1) $\theta_{\ell,k}$ is assigned to one of the previously selected cluster $\theta^\star$ with probability $\Pi_1$ 2)$\theta_{\ell,k}$ is not yet chosen by the previously transitioned to time $k$ objects, $\theta_{\ell,k}$  should be transitioned from time $(k-1)$ with transition kernel $\theta_{\ell,k} \sim \nu(\theta^\star_{\ell,k-1},\cdot)$ and then 
Proof of theorem \ref{thm1} immediately follows the cases (1)-(3).

\section{Learning Model}
\label{ddplearning}
The DDP, as discussed in Algorithm \ref{DDP_alg1}, provides a prior on the object state parameter distributions at time step $k$. This estimate may be updated using the 
available measurement vectors,  $\mathcal{Z}_k = \{\bz_{l, k}$, $l\eqq 1, \ldots, M_k\}$. The posterior distribution is then used to estimate the trajectory of objects and find the time-dependent object cardinality. It is assumed that each measurement is independent of each other and only generated from one object. Theorem \ref{thm1} implies that we may exploit Dirichlet process mixtures to estimate the density of the measurements and cluster them. Note that the measurement vectors are unordered meaning 
 the $l$th measurement is not necessarily associated to the $\ell$th object state,  $l\neq \ell$. As the DDP is used to label 
 the object states at time step $k$, the Dirichlet process mixtures can be used to learn and assign  a measurement to its associated object  identity.
In order to create the mixtures of distributions,  we use the DDP prior in Algorithm \ref{DDP_alg1}.  The Mixing measure is drawn from the generated DDP in order to 
 to infer the likelihood distribution $p(\bz_{l,k}$$\mid$$\bth_{\ell,k}, \bx_{\ell, k})$ and update the object state estimates .
In particular,  $p(\bz_{l,k}$$\mid$$\bth_{\ell,k}, \bx_{\ell, k})$ is inferred from 
\begin{flalign}
\label{likelihood}
&{\bf{\theta_{\ell,k}}} \sim \text{DDP}(\alpha , H)\notag \\
&\bx_{\ell,k} \mid \bths_{\ell, k} \sim F({\bf{\theta^\star_{\ell,k}}})\\
&\bz_{l,k} | {\bf{\theta^\star}}_{\ell,k}, \bx_{\ell,k} \sim R(\bz_{l,k}|{\bf{\theta^\star}}_{\ell,k}, \bx_{\ell,k} ) \notag
\end{flalign}
where $F(\bth_{\ell,k})$ is a distribution whose  density  
 follows \ref{state1}, \ref{state2}, \ref{state3}, and $R(\bz_{l,k}|\bth_{\ell,k}, \bx_{\ell,k})$ is a distribution that 
depends on the measurement likelihood function. %
 Algorithm \ref{DDP_alg2} summarizes our implementation of 
this mixing process to cluster the measurements and track the objects. 
Algorithms \ref{DDP_alg1} and  \ref{DDP_alg2},  together with MCMC sampling
 methods, constitute  the overall DDP-EEM multiple object 
 tracking algorithm based on the DDP.
Sampling in both algorithms is performed using 
MCMC methods; in particular, we use Gibbs sampling.
\begin{algorithm} [t]
\begin{algorithmic}
\caption{Infinite Mixture Model to Cluster and  Track Objects}
\label{DDP_alg2}
\STATE {\textbf{Input}:} Measurements: $\{\bz_{1,k}, \ldots, \bz_{M_k, k}\}$
\STATE {\textbf{Output}:} $N_k$,  \text{cluster configurations}, and \text{posterior distributions}
\STATE From construction of prior distribution
\STATE {\textbf{At time} k}
\FOR{$ \ell=1$ \TO $N_k$}
\STATE Sample  $\{\bth_{1,k}, \ldots, \bth_{N_k, k}\}$ 
and $\{\bx_{1,k}, \ldots, \bx_{N_k, k}\}$ as in Algorithm \ref{DDP_alg1}
\ENDFOR
\FOR {$l=1$ \TO $L_k$}
\STATE Draw $\bz_{l,k}$$\mid$$\bx_{\ell,k}, \bth_{\ell,k}$ from \ref{likelihood} %\\
%\hspace*{0.3in} 
% $F(\bths_{\ell,k}, \bx_{\ell,k} )$ in \eqref{likelihood}
%$\propto$$p_k(\cdot \mid 
%\bx_{\ell,k}) \, \text{DDP}(\alpha,G_0)$ in \eqref{likelihood}
\ENDFOR
\RETURN $\mathcal{C}_k$ : induced cluster assignment indicators
\STATE \textbf{Update:} $\mathcal{CA}_{k} = \mathcal{CA}_{k-1}\cup \mathcal{C}_k$: set of cluster assignments up to time $k$
\RETURN{$N_k$, $\mathcal{CA}_{k}$, and posterior of $\bz_{l,k}$$\mid$$\bx_{\ell,k},\bth_{\ell,k}$}
\end{algorithmic}
\end{algorithm}

\begin{comment}
\begin{algorithm} [t]
\begin{algorithmic}
\caption{Dirichlet process mixtures to cluster \\measurements and track objects}
\label{DDP_alg2}
\STATE {Input}: Measurements: $\{\bz_{1,k}, \ldots, \bz_{M_k, k}\}$
\STATE {Output}: $N_k$,  \textit{cluster configurations}, and \textit{posterior}
\STATE {\bf{At time} k}
\FOR{$ \ell=1$ \TO $N_k$}
\STATE Sample  $\{\bth_{1,k}, \ldots, \bth_{N_k, k}\}$ 
and $\{\bx_{1,k}, \ldots, \bx_{N_k, k}\}$ from DDP in Algorithm \ref{DDP_alg1}
\ENDFOR
\FOR {$l=1$ \TO $L_k$}
\STATE Sample $\bz_{l,k}$$\mid$$\bx_{\ell,k}, \bth_{\ell,k}$ using  \ref{likelihood} %\\
%\hspace*{0.3in} 
% $F(\bths_{\ell,k}, \bx_{\ell,k} )$ in \eqref{likelihood}
%$\propto$$p_k(\cdot \mid 
%\bx_{\ell,k}) \, \text{DDP}(\alpha,G_0)$ in \eqref{likelihood}
\ENDFOR
\RETURN{$N_k$ and posterior of $\bz_{l,k}$$\mid$$\bx_{\ell,k},\bth_{\ell,k}$}
\end{algorithmic}
\end{algorithm}
\end{comment}
\subsection{Bayesian Inference: Gibbs Sampler}
\label{gib}
Identifying the labels for an object tracking problem and estimating the density parameters using DDP is a state-of-the-art method. 
However, computing the explicit posterior and therefore, the trajectory can be troublesome. The development of MCMC methods to sample form the posterior 
distribution has made this issue computationally feasible. The Gibbs sampler is an MCMC method to sample from the density, without directly requiring the density, 
by using the marginal distributions. The Gibbs sampler provides sample from the posterior distribution from the finite dimensional representation rather than sampling from  infinite dimension representations where one can use slice sampling methods. %by using a Markov chain that has the posterior as its equilibrium distribution.

We outline the Gibbs sampler inference scheme for our method. We use a Gibbs sampling technique to iterate between sampling the state variables and the set of dynamic DDP parameters. We propose a method that can handle conjugate prior. This method can be generalized to a non-conjugate prior \cite{neal2000}. A key feature of this modeling is the discreetness of the DDP \cite{Mac2000, Mac1998}. We assume that conjugate priors are used; however, one can easily generalize this to non-conjugate priors. We outline this scheme next.

{\bf{Predictive Distribution:}} The Bayesian posterior can be solved through the following:
\begin{equation}
\label{predictive}
P(\bx_{\ell,k}|\mathcal{Z}_k) = \int_{\theta}{P(\bx_{\ell,k}|\mathcal{Z}_k,\theta) dG(\theta|\mathcal{Z}_k)}
\end{equation}
where $G(\theta|\mathcal{Z}_k)$ the posterior distribution of the parameters given the observations. Note that, with respect to the predicting $\bx_{\ell,k}$, we have $P(\bx_{\ell,k}|\mathcal{Z}_k,\theta) = P(\bx_{\ell,k}|\theta)$ and can be evaluated as follows:
\begin{equation}
P(\bx_{\ell,k}|\Theta) = \int{P(\bx_{\ell,k}|\theta_{\ell,k}) d\pi(\theta_{\ell,k}|\Theta)}
\end{equation}
where $\pi(\theta_{\ell,k}|\Theta)$ is posterior distribution of $\theta_{\ell,k}$ given the rest of parameters. The distribution of $\pi(\theta_{\ell,k}|\Theta)$ is given by:
\begin{equation}
\label{posttheta}
\pi(\theta_{\ell,k}|\Theta) = \sum\limits_{\theta\in\Theta_k - \{\theta_{\ell,k}\}} \Pi^1_{j,k} \delta_{\theta}(\theta_{\ell,k}) + \sum\limits_{\substack{\theta \in \Theta^\star_{k|k-1}\setminus \Theta\\ \theta\neq \theta_{\ell,k}}}\Pi^2_{j,k} \nu({\bf{\theta}^\star_{\ell, k-1}},{\bf{\theta_{\ell, k}}}) \delta_{\theta}({\bf{\theta_{\ell,k}}}) + \Pi^3_k H(\theta_{\ell,k}).
\end{equation}

To compute the \ref{predictive}, we need to calculate the posterior $G(\theta|\mathcal{Z}_k)$. However, direct computation of \ref{predictive} is extremely computationally expensive due to the complexity of $G(\theta|\mathcal{Z}_k)$ \cite{antoniak1974}.We propose a Gibbs sampling approximation of this distribution. The following distribution is obtained by combining the prior with the likelihood in order to use for Gibbs sampling:

\begin{flalign}
\label{Gibbs}
\bth_{\ell,k}\mid\bth_{-\ell,k}, \mathcal{Z}_k \sim& 
 \sum\limits_{j=1 }^{ \mid \mathcal{C}_k\mid} \zeta_{j,k} \; \delta_{\bth_{j,k}}(\bth_{\ell,k}) 
+  \sum\limits_{\substack{ j=1 \\  j \notin {\mathcal{C}}_k }}^{D_{k\mid k-1}}\beta_{j,k}  \;  K_{j,k}(\bth_{\ell,k}) 
+  \gamma_{\ell,k}\, H_{\ell}(\bth_{\ell,k}),
\end{flalign}
where $\theta_{-\ell,k}$ by convention is the set $\{\theta_{j,k} \textit{, } j \neq \ell\}$. It is shown in appendix that
\begin{flalign}
&\zeta_{j,k} = \frac{\left[V_k\right]_j+  \sum\limits_{i=1}^{D_{k|k-1}} \left[V^\star_{k|k-1}\right]_i \lambda_{i, k \mid k-1}   \delta_i(c_{j, k})}{g_{\ell-1, k-1}}R(\bz_{\ell,k}| \bx_{j,k},\bth_{j,k})\notag\\
& \beta_{j,k} = \frac{\sum\limits_{\substack{ i=1 \\  i \notin \mathcal{C}_ k}}^{D_{k\mid k-1}}\left[V^\star_{k|k-1}\right]_j \lambda_{j,k|k-1}}{g_{\ell-1, k-1}} \\
&  \sum\limits_{j=1}^{\mid \mathcal{C}_k\mid} \zeta_{j,k} + \sum\limits_{\substack{ j=1 \\  j \notin \mathcal{C}_k}}^{D_{k\mid k-1} }\beta_{j,k} + \gamma_{\ell, k} = 1\notag
\end{flalign}
where $g_{\ell-1, k-1} = (\ell-1)+ \sum_{i=1}^{D_{k| k-1}} \left[V^\star_{k|k-1}\right]_i \lambda_{i, k|k-1} + \alpha$,  $\alpha$$>$$0$. Moreover, $K_{i,k} = R(\bz_{\ell,k}| \bx_{j,k},\bth_{j,k})$ and $dH_{\ell}(\theta) \propto R(\bz_{\ell,k}| \bx_{j,k},\theta) dH(\theta) $ where $H$ is the base distribution on $\theta$.

\begin{proof}
The proof is provided in Appendix \ref{proof1}.
\end{proof}

\subsection{Convergence of DDP-EMM through Gibbs Sampler}
There are many sets of conditional distributions that can be used as the basis of Gibbs sampler for which violate the required posterior convergence conditions of the sampler. In this section, we discuss conditions under which the proposed Gibbs sampler in section \ref{gib} converges to the posterior distribution.  The result mainly depends on the Theorems in \cite{tierney1994}. 

We first prove that the regardless of initial condition the transition kernel converges to the posterior for almost all initial condition and then we provide the set of conditional distributions to guarantee the convergence to the posterior of the introduced Markov chain using  Theorem 1 in \cite{tierney1994}. To this end, let $\text{K}(\theta_0, \Theta)$ and $P_{\theta}(\cdot|\mathcal{Z}_k)$ be the transition kernel for the Markov chain starting at $\theta_0$ and stopping in the set $\Theta$ after one iteration of the algorithm introduced in section \ref{gib} and the posterior distribution of parameters given the observations at time $k$, respectively. 

\begin{theorem}
\label{thm2}
At each time step $k$, convergence to the posterior distribution $P_\theta(\cdot|\mathcal{Z}_k)$ does not depend on the starting value, i.e.,
\begin{equation}
||\text{K}^n_k(\theta_0, \cdot) - P_\theta(\cdot |\mathcal{Z}_k)||_{TV} \longrightarrow 0
\end{equation}
as $n\rightarrow \infty$, for almost all initial conditions $\theta_0$ in total variation norm. 
\end{theorem}
This theorem guarantees the convergence to the posterior for almost all initial values. The proof is provided in Appendix \ref{proofthm2}. This result specifically holds if normal distribution is considered \cite{escobar1995, escobar1994}.

\section{Properties of DDP-EMM}
\label{prop}
Given the configurations at time $(k-1)$, the infinite exchangeable random partition induced by $\mathcal{C}_k$ at time $k$ follows the exchangeable partition probability function (EPPF) \cite{aldous1985}
\begin{equation}
\label{EPPFDir}
p(\left[V_k\right]^\ast_1, \dots, \left[V_k\right]^\ast_{D_k}) = \frac{\alpha^{D_k}}{\alpha^{\left[N_k\right]}}\prod\limits_{j=1}^{D_k}(\left[V_k\right]^\ast_j-1)!
\end{equation}
where $D_k$ is the number of unique cluster parameter,  $\left[V_k\right]^\ast_{j}, \hspace{0.1cm} j = 1, \dots, D_k$ is the cardinality of the cluster $c_{j,k}$, and  $\alpha^{\left[n\right]} = \alpha (\alpha+1) \dots (\alpha+n-1)$. Note that number of the objects at time $k$, $N_k$, plays an important rule in partitioning. Also, due to variability of $N_k$ at time $k$, the relationship between partitions based on $N_k-1$ and $N_k$ is important. The EPPF of the infinite random exchangeable partition  based on the partition on $N_k$ and $(N_k-1)$ objects given the configuration at time $(k-1)$ satisfies
\begin{equation}
\label{partition}
p_{N_k-1}(\left[V_k\right]^\ast_1, \dots, \left[V_k\right]^\ast_{D_k}) = \sum\limits_{j = 1}^{D_k} p_{N_k}(\left[V_k\right]^\ast_1, \dots,\left[V_k\right]^\ast_j+1 ,\dots \left[V_k\right]^\ast_{D_k}) + p_{N_k}(\left[V_k\right]^\ast_1, \dots, \left[V_k\right]^\ast_{D_k}, 1).
\end{equation}
The equation \ref{partition} entails a notion of consistency of the partitions in the distribution sense. The equation \ref{partition} holds due to the Markov property of the process given the configuration at time $(k-1)$. 
{\subsection{Consistency}}
 Suppose $\mathcal{Z}_k = \{\bz_{1,k}, \dots, \bz_{M_k,k}\}$ is the collection of $M_k$ measurements at time $k$ with joint conditional distribution $R(\mathcal{Z}_k|\theta,\bX_k)$ with respect to the product probability space which is indexed by $\theta\in \Theta$, where $\Theta$ is a first countable topological space. Let $r_{\theta}(\mathcal{Z}_k|\bX_k)$ be the density corresponding to the probability measure $R(\mathcal{Z}_k|\theta,\bX_k)$.
 
{\textit{Definition}}: The posterior distribution $P_{\theta}(\cdot | \mathcal{Z}_k)$ is \textit{weakly consistent} at true parameters $\theta_0 \in \Theta$ at each time step $k$ if $P_{\theta}(\bf{U}_k | \mathcal{Z}_k)\to 1$ in $r_{\theta_0}(\mathcal{Z}_k|\bX_k)$-probability as $n \to \infty$ for every neighborhood $\bf{U}_k$ of true parameters $\theta_0$.

 {\textit{Definition}}: The posterior distribution $P_{\theta}(\cdot | \mathcal{Z}_k)$ is \textit{strongly consistent} at true parameters $\theta_0 \in \Theta$, if the convergence is almost sure. 
\subsubsection{Posterior Consistency of the Model} 
\label{ddpconsis}
In section \ref{DDP_Model}, we introduced a general model such that the distribution over the parameters at time $k$ conditioned on the configurations at time $(k-1)$ is a Dirichlet process. Schwartz \cite{schwartz1964} and Ghosal, et. al. \cite{ghosal1999} discussed the weak and strong consistency of the posterior distribution for a general kernel under a DDP prior. The main result on weak consistency is due to Schwartz theorem. Let $r_{\theta_0}$ be the true density of observations with corresponding probability measure $R_{\theta_0}$, 
\begin{prop}[Schwartz 1965]
If $r_{\theta_0}$ is in the KL support {\footnote{Density $r_{\theta_0}$ is in KL support of the prior $P_k$ if for any $\epsilon>0$, $P_k(r_{\theta}: \int r_{\theta_0} \log{\frac{ r_{\theta_0}}{r_\theta}}<\epsilon) > 0$ and is denoted by $r_{\theta_0} \in KL(\epsilon, P_k)$.}} of the prior distribution $P_k$ on the topological space of all parameters with an appropriate $\sigma$-field, $r_{\theta_0} \in KL(\epsilon, P_k)$, then posterior distribution $P_{\theta}(\cdot | \mathcal{Z}_k)$ is weakly consistent at $r_{\theta_0}$.
\end{prop}
Theorem \ref{cons} hence deals with the consistency of the posterior at time $k$ under the prior distribution conditioned on the previous time step $(k-1)$ introduced in equation \ref{DDP-EMM}.
\begin{theorem}
\label{cons}
Let the true density be $r_{\theta_0}$ and $P_k$ be the prior distribution at time $k$ conditioned on the configurations at time $(k-1)$ given by \ref{DDP-EMM}, if $r_{\theta_0}$ is in the support of $P_k$, then $P_k(KL(\epsilon, r_{\theta_0} ))>0$ and therefore, the posterior is weakly consistent. 
\end{theorem}
Proof of this theorem is straightforward and aligns with the proof in \cite{ghosal1999}. Intuitively speaking, one can prove this theorem by drawing an arbitrary measure from the base and show that the condition in the theorem holds for the set $KL(\epsilon, r_{\theta_0})$. It is worth mentioning, $P_k(KL(\epsilon, r_{\theta_0} ))>0$ is not a tight condition and holds true for many nonparametric models. In particular, in the case of Gaussian kernel, this condition is satisfied and hence the posterior is consistent using Gaussian kernels (Theorem 3, \cite{ghosal1999}). 

{\bf{Remark:}} Note that $r_{\theta_0}$ being in the support of $P_k$ is equivalent to $\text{support}(r_{\theta_0})\subset \text{support}\big(\sum\limits_{\Theta_k} \Pi^1_{j,k} \delta_{\theta_{\ell,k}} + \sum\limits_{\Theta^\star_{k|k-1}\setminus \Theta_{k}}\Pi^2_{j,k} \nu({\bf{\theta}^\star_{\ell, k-1}},{\bf{\theta_{\ell, k}}}) \delta_{{\bf{\theta_{\ell,k}}}} + \Pi^3_k H\big)$, provided $\Pi^1_{j,k} , \Pi^2_{j,k}, \text{and}\hspace{0.1cm} \Pi^3$ as equation \ref{DDP-EMM}. 

{\bf{Remark:}} The posterior is also strongly consistent due to Theorem 1 of \cite{barron1999}. 

{\subsection{Posterior Contraction Rate of the Model}}

Posterior contraction rate discusses how fast the posterior distribution approaches the true parameters from which the observations are generated. The contraction rate is highly related to posterior consistency.

 {\textit{Definition}}: A sequence $\epsilon_n$ is posterior contraction rate at the parameter $\theta_0$ with respect to a metric $d$ if for every sequence $C_n\to\infty$, we have $P_\theta(\theta: d(\theta, \theta_0) \geq C_n\epsilon_n|\mathcal{Z}_k)\to0$ in $P_{\theta_0}$-probability as $n\to\infty$. 
 
 The following theorem specifies the contraction rate of the posterior contraction of the DDP based model introduced in section \ref{DDP_Model}. Assume that each $\bz_{j,k} \in \mathbb{R}^{n_z}, \hspace{0.1cm} j = 1,\dots, M_k$. We denote $N_{\left[\right]}(\epsilon, \mathcal{H}_{\kappa}(\left[0,1\right]^{n_z}), d)$ to be the $\epsilon$-bracketing number of Holder space $\mathcal{H}_{\kappa}$ with $\kappa$ degree of smoothness on the compact space of $\left[0,1\right]^{n_z}$ with respect to the distance $d$.
 
 \begin{theorem}
 \label{rate}
 Suppose $\mathcal{P}$ is the set of all distributions where the square root of the density belongs to the Holder space $\mathcal{H}_{\kappa}(\left[0,1\right]^{n_z})$. Let $\epsilon_n$ be a decreasing sequence such that $\log N_{\left[\right]}(\epsilon,\mathcal{P}, d_H) \leq n\epsilon^2_n$ and $n\epsilon^2_n/\log n \to 0$, where $d_H$ is Hellinger distance{\footnote{$d_H(p,q) = (\int (\sqrt{p}-\sqrt{q})^2 d\mu)^{\frac{1}{2}}$ is the Hellinger distance given the dominating measure $\mu$.}}. Then, the posterior distribution at time $k$ of the DDP prior given $\mathcal{Z}_k$ and the previous time $(k-1)$ configurations converges to the true density at the rate of $\epsilon_n$, where $\epsilon_n$ is the order of $n^{-\frac{\kappa}{2\kappa+n_z}}$.
 \end{theorem}
 {\bf{Remark:}} Note that the rate in Theorem \ref{rate} matches the minimax rate for density estimators. Hence, the DDP prior constructed through this model achieves the optimal frequentist rate.
 
 \begin{proof}
 The proof follows the theorem 3.1 in \cite{ghosal2007}. Ghosal et.al proved that $\epsilon_n$ satisfying the conditions in the theorem is indeed the contraction rate. 
Define $N(\epsilon, \mathcal{H}_{\kappa}(\left[0,1\right]^{n_z}), ||\cdot||_{\infty})$ to be the $\epsilon$-covering number of $\mathcal{H}_{\kappa}(\left[0,1\right]^{n_z})$ with respect to supremum norm. Since one can find the $\left[l,u\right]$ bracket from the uniform approximation, the bracketing number with Hellinger distance grows with the same rate as the $\epsilon$-covering number with supremum norm. Therefore, it is enough to find an upper bound for $N(\epsilon, \mathcal{H}_{\kappa}(\left[0,1\right]^{n_z}), ||\cdot||_{\infty})$.
\begin{lemma}[Kolmogorov, Tihomirov 1961\cite{tikhomirov1993}]
\label{KT}
For $\left[0,1\right]^{n_z} \subset \mathbb{R}^{n_z}$, there exist Constants $C$ depending on $\kappa$ and $n_z$ such that for every $\epsilon>0$, we have
\begin{equation}
\log N(\epsilon, \mathcal{H}_{\kappa}(\left[0,1\right]^{n_z}), ||\cdot||_{\infty}) \leq C\big(\frac{1}{\epsilon})^{\frac{n_z}{\kappa}}
\end{equation}
\end{lemma}
Lemma \ref{KT} implies that $\log N_{\left[\right]}(\epsilon,\mathcal{P}, d_H) \leq C\big(\frac{1}{\epsilon})^{\frac{n_z}{\kappa}}$ and thus the convergence rate is the order of $n^{-\frac{\kappa}{2\kappa+n_z}}$.
 \end{proof}

\section{Nonparametric MMT: Dependent Two-Parameter Poisson Dirichlet Process Construction}
\label{DPY}
We thus far introduced the dependent Dirichlet process model to incorporate a learning algorithm as a prior over the time evolving object state distribution based on the measurements. When using the Dirichlet process to model the transitioning of objects into clusters, the expected number of unique clusters varies exponentially according to $\alpha log(N)$, where $\alpha$ is the concentration parameter and $N$ is the total number of objects to be clustered. A more flexible model is offered by the two parameter Poisson-Dirichlet process, Pitman-Yor process, as, in this case, an additional discount parameter, $0 \leq d < 1$, with $\alpha > -d$, is used to control the number of clusters in the model. Specifically, with the Pitman-Yor process model, the expected number of unique clusters varies according to the power-law $\alpha N^d$ \cite{teh2006}. Following the power-law, the higher the number of unique (non-empty) clusters, the higher the probability of having even more unique clusters. Also, clusters with only a small number of objects have a lower probability of having new objects. This more flexible model offered by the Pitman-Yor process is a better match for the tracking problem with a time-varying number of objects. With a maximum number of $N_k$ objects at time step $k$, an object may stay in the scene from the previous time step, leave the scene, or enter the scene for the first time. Thus, the object state would benefit from a larger number of available clusters to ensure all dependencies are captured.

In order to also capture time evolution, we introduce a family of dependent Pitman-Yor (DPY) processes that can be used to model a collection of random distributions that are related but not identical. As a result, we utilize the DPY to model the multiple object state prior distributions by directly incorporating learning multiple parameters from correlated information. The resulting DPY state transitioning prior (DPY-STP) method formulates the state transition such that the object cardinality at time step $k$ is dependent on its value at the previous time step $(k - 1)$. Also, the index assigned to the cluster that contains an object state is dependent on the cluster indexing of the previously clustered object states at the same time step $k$. If a new object enters the scene, its state must be modeled without knowledge on the expected number of objects. We outline the detail for this approach in the following. 

\subsection{DPY-STP Algorithm Construction for State Transitioning}
In this section, we introduce an evolutionary time dependent model to multiple object tracking based on our proposed dependent Pitman-Yor (DPY) process to learn object labels. The advantage of this model over the DDP-EMM method introduced in section \label{DDP} is that this approach introduces a dependent Pitman-Yor (DPY) process that marginally preserves the Pitman-Yor process and therefore, it allocates higher probability to unique clusters and therefore it is a better fit for multi object tracking problem. 
In particular, our approach directly incorporates learning multiple parameters through related information, including object labeling at the previous time step or labeling of previously considered objects at the same time step. In particular, the choice of the DPY as a prior on the object state distributions is based on the following dynamic dependencies in the state transition formulation: (A) the number of objects present at time step $k$ relies on the number of objects that were present at the previous time step $(k-1)$, (B) the clustering index of the parameter state of the $\ell$th object at time step $k$ depends on the clustering index of the state parameters of the previous $(\ell-1)$ objects at the same time step $k$, and (C)  model a new object entering the scene without requiring any prior knowledge on the expected number of objects. The DPY-STP method we use to model the state transition process, accounting for multiple dependencies, is discussed next and is summarized in Algorithm \ref{alg3}. In particular, we provide: (a) the information available at time step $(k - 1)$, (b) how this information transitions from time step $(k - 1)$ to time step $k$, and (c) how the DPY-STP model is constructed at time step $k$ to estimate the object density.
 
\textit{ {\bf{Available parameters at time $(k-1)$:}}}  The DPY-STP algorithm assumes that the following parameters are available from previous time steps at time $(k-1)$:

\begin{itemize}
\item Let ${\bf{X}_{k-1}} = \{{\bf{{x_\ell,k-1}}}: \ell = 1,\dots, N_{k-1}\}$ be the object states at time $(k-1)$.
\item Let $\mathcal{CA}_{k-1} = \{\mathcal{C}_1,\dots, \mathcal{C}_{k-1}\}$ be the cluster assignment up to time $(k-1)$, where $\mathcal{C}_J = \{c_{1,J},\dots, c_{N_J,J}\}$ is the cluster assignments at time step $J$.
\item $\Theta_{k-1} = \{{\bf{\theta_{\ell,k-1}}} :  \ell = 1,\dots, N_{k-1}\}$ is the set of object state parameters available at time $(k-1)$ associated with $\mathcal{C}_{N_{k-1}}$(note that $\theta_\ell$'s are not necessarily unique).
\item Let $\Theta^\star_{k-1} = \{{\bf{\theta^\star_{\ell,k-1}}} :  \ell = 1,\dots, D_{k-1}\}\subset \Theta_{k-1}$ be the set of unique parameters, and $D_{k-1}$ be the number of uniques parameters. 
\item Define ${\bf{V}}^\star_{k-1}$ to be a vector of size $D_{k-1}$ containing the size of non empty clusters associated with $\mathcal{C}_{k-1}$. One can include empty clusters and define the size of this vector to be $N_{k-1}$. However, it is computationally more efficient to exclude size zero clusters.   
\end{itemize}

\textit{{\bf{Available parameters transitioning from time $(k-1)$ to time $k$}}:} Assume $s_{\ell,k|k-1}$ associate with the $\ell$th object at time $(k-1)$ has a Bernoulli distribution with parameter $P_{\ell,k|k-1}$, $s_{\ell,k|k-1} \sim \text{Bernoulli}(P_{\ell,k|k-1})$. Given $s_{\ell,k|k-1}$, the object $\bx_{\ell,k-1}$ leaves the scene with probability $1- P_{\ell,k|k-1}$ or remains in the FOV with probability $P_{\ell,k|k-1}$ and transitions to a new state with the Markov transition kernel $\mathbb{Q}_{\underline{\theta}}(\bf{x_\ell(k-1)}, \cdot)$. We assume if all the objects in a cluster leave the scene the cluster no longer exist. Let $\Theta^\star_{k|k-1}$ be the set of unique parameters at time $(k-1)$ that are transitioned to time step $k$.  We define ${\bf{V^\star}}_{k|k-1}$ to be the vector of size of $D_{k-1}$ containing the size of each cluster after transitioning to time $k$. it is worth mentioning that a cluster with size zero implies that the cluster no longer exists. To keep track of the survived objects, let $\mathcal{CS}_{k|k-1}$ be the cluster survival indicator defined as $$\mathcal{CS}_{k|k-1} = \{ \eta_{1,k|k-1},\dots,\eta_{D_{k-1}, k|k-1}\}$$ where $\eta_{j,k|k-1} = 0$ corresponds to disappearance of the $j$th cluster and $\eta_{j,k|k-1}=1$ implies that there is at least one element in the $j$th cluster. 

\textit{{\bf{DPY Prior Construction at time $k$}}:} Each survived cluster (a cluster with non-zero size) is updated through a transition kernel. Assume that the cardinality of $\ell$th cluster at time $(k-1)$ is still non-zero after transitioning, the $\ell$th object parameter will evolve according to the following transition kernel:
\begin{equation}
\label{trantheta}
{\bf{\theta_{\ell, k}}} \sim \zeta({\bf{\theta^\star_{\ell, k-1}}}, \cdot)
\end{equation}
Let ${\bf{\theta}_{\ell, k}}$ be the transitioned $\ell$th state object parameter at time $k$, we construct the dependent Pitman-Yor prior as follows: 
\begin{description}
\item[Case 1:] The $\ell$th object belongs to one of the survived and transitioned clusters from time $(k-1)$ and occupied at least by one of the previous $\ell-1$ objects. The object selects one of these clusters with probability:
\small
\begin{flalign}
\label{eq1}
\Gamma^1_{j,k} (\text{Choosing jth cluster} | {\bf{\theta}_{1,k}},\dots, {\bf{\theta}_{\ell-1,k}}) = \frac{ \sum \limits_{i=1}^{D_{k-1}}\left[V^\star_{k|k-1}\right]_i \eta_{i, k \mid k-1}  \delta_i(c_{j, k})+ \left[{\bf{V}}_k\right]_j - d}{\sum\limits_{j=1}^{\ell-1}\sum \limits_{i=1}^{D_{k-1}}\left[V^\star_{k|k-1}\right]_i \eta_{i, k \mid k-1}  \delta_i(c_{j, k}) + \sum\limits_{j=1}^{\ell-1}\left[{\bf{V}}_k\right]_j + \alpha }
\end{flalign}
\normalsize
where $\left[{\bf{V_k}}\right]_j$ indicates the $j$th element of vector ${\bf{V_k}}$ at time $k$, $0 \leq d < 1$ and $\alpha > -d$ are  the discount and strength parameters in the Pitman-Yor process, respectively.

\item[Case 2:] The $\ell$th object belongs to one of the survived and transitioned clusters from time $(k-1)$ but this cluster has not yet been occupied by any one the first $\ell-1$objects. The object selects such a cluster with probability:
\small
\begin{flalign}
\label{eq2}
\Gamma^2_{j,k} (\text{Choosing jth cluster that has not been selected yet}& | {\bf{\theta}_{1,k}},\dots, {\bf{\theta}_{\ell-1,k}}) =\notag\\
& \frac{ \sum \limits_{i=1}^{D_{k-1}}\left[V^\star_{k|k-1}\right]_i \eta_{i, k \mid k-1}  \delta_i(c_{j, k}) - d}{\sum\limits_{j=1}^{\ell-1}\sum \limits_{i=1}^{D_{k-1}}\left[V^\star_{k|k-1}\right]_i \eta_{i, k \mid k-1}  \delta_i(c_{j, k}) + \sum\limits_{j=1}^{\ell-1}\left[{\bf{V}}_k\right]_j + \alpha }
\end{flalign}
\normalsize
\item[Case 3:]The object does not belong to any of the existing clusters, thus a new cluster parameter is drawn from some base distribution $H$, corresponding to the base distribution in Pitman-Yor process, with probability:
\small
\begin{flalign}
\label{eq3}
\Gamma^3_k (\text{Creating new cluster} | {\bf{\theta}_1(k)},\dots, {\bf{\theta}_{\ell-1}(k)}) = \frac{|D_k|_{\ell-1}d+\alpha}{\sum\limits_{j=1}^{\ell-1}\sum \limits_{i=1}^{D_{k-1}}\left[V^\star_{k|k-1}\right]_i \eta_{i, k \mid k-1}  \delta_i(c_{j, k}) + \sum\limits_{j=1}^{\ell-1}\left[{\bf{V}}_k\right]_j + \alpha } 
\end{flalign}
\normalsize
where $|D_k|_{\ell-1}$ is the total number of the clusters at time $k$ created by the first $(\ell - 1)$ objects. 
\end{description}
In the above construction, $\Gamma^1_{j,k}, \Gamma^2_{j,k}, \text{and }  \Gamma^3_k$ are the probability of selecting an object cluster or creating a new object cluster. The temporal dependency among the objects follows a dependent Pitman-Yor process where the marginal distribution is a Pitman-Yor Process. This property makes this process easy to implement.  The following theorem summarizes this property:
\begin{theorem} 
\label{thm3}
Suppose that the space of state parameters is separable and complete metrizable space. The process defined by probabilities \ref{eq1}, \ref{eq2}, and \ref{eq3} defines a Pitman-Yor process at each time step given the previous time configurations, i.e., 
\begin{equation}
\label{condd}
\textit{DPY-STP}_k | \textit{DPY-STP}_{k-1} \sim \mathcal{PY}\Big(d,\alpha, \sum_{\Theta_k} \Gamma^1_{j,k} \delta_{\theta_{\ell,k}} + \sum_{\Theta^\star_{k|k-1}\setminus \Theta_k}\Gamma^2_{j,k} \zeta({\bf{\theta}^\star_{\ell_k-1}},{\bf{\theta_{\ell, k}}}) \delta_{{\bf{\theta_{\ell, k}}}} + \Gamma^3_k H\Big).
\end{equation}
\end{theorem}
where $\delta_{\theta}(\Theta) = 1$ if $\theta \in \Theta$ and $\delta_{\theta}(\Theta) = 0$, if $\theta \notin \Theta$. We eliminate the proof of this theorem since it is the direct result of cases (1)-(3). Given the conditional distribution \ref{condd}, theorem \ref{thm4} provides an object density estimator. 
\begin{theorem} 
\label{thm4}
Assume the space of states, $\mathcal{X}$, is separable and complete metrizable topological space, given equations (\ref{eq1})-(\ref{eq3}) state distribution is estimated as follows:
\small
\begin{equation}
\label{density}
        p({\bf{x_{\ell, k}}} | {\bf{x}_{1,k}},\dots, {\bf{x}_{{\ell-1},k}}, {\bf{X}}_{k|k-1}, \Theta^\star_{k|k-1}, \Theta_k)=
        \begin{cases}
           \mathbb{Q}_{\underline{\theta}}(\bf{x_{\ell, k-1}}, {\bf{x_{\ell, k}}})  f({\bf{x_{\ell, k}}} | {\bf{\theta^\star_{\ell,k}}}) &\text{ If case 1} \\
           \mathbb{Q}_{\underline{\theta}}(\bf{x_{\ell,k-1}}, {\bf{x_{\ell,k}}})  \zeta({\bf{\theta^\star_{\ell, k-1}}},{\bf{\theta^\star_{\ell, k}}})  f({\bf{x_{\ell,k}}} | {\bf{\theta^\star_\ell(k)}}) &\text{ If case 2}\\
            \int_{\bf{\theta}}{f({\bf{x_{\ell,k}}} | \theta)} dH(\theta) & \text{ If case 3}
        \end{cases}
    \end{equation}
\normalsize
for some density $f(\cdot | \theta)$, distribution $H$ on parameters, and ${\bf{X}}_{k|k-1}$ the set of survived state objects. Note that elements of $\Theta_k$ are chosen with probability $\Gamma^i$, $i = 1, 2, 3$.
\end{theorem}
\begin{proof} (Sketch of proof)
The proof is immediately resulted from the problem statement. We provide an intuition for this theorem. From case (1): ${\bf{x_{\ell,k-1}}}$ transitions to time $k$ according to the Markov transition kernel $\mathbb{Q}_{\underline{\theta}}(\bf{x_{\ell, k-1}}, \cdot)$ and then is assigned to one of the existing clusters that is already used by one of the objects. From case (2): ${\bf{x_{\ell,k-1}}}$ and the cluster parameter ${\bf{\theta^\star_{\ell,k-1}}}$ transition to time $k$ according to Markov transition kernels $\mathbb{Q}_{\underline{\theta}}(\bf{x_{\ell, k-1}}, \cdot)$ and $\zeta({\bf{\theta^\star_{\ell, k-1}}}, \cdot)$, respectively,  and therefore the object is assigned to the this new cluster. From case (3): new object does not belong to any of the previously assigned clusters, i.e., a new object emerges to the scene. In this case, we generate a new parameter from the base distribution $H$ and assign the object to the newly created cluster. 
\end{proof}

\begin{algorithm}[t]
\caption{DPY-STP model for state transition process.}
\label{alg3}
\begin{algorithmic}%[1]
%
\begin{comment}
\STATE {\bf At time} $(k-1)$:
\STATE \hspp $\bX_{k-1}\eqq  \{ \bx_{\ell, k-1} \ \ldots \   \bx_{N_{k-1}, k-1} \}$, object states
%
\STATE \hspp $\cca_{N_{k-1}}\eqq [ c_1, \ c_2, \ \ldots, \ c_{N_{k-1}} ]$, cluster assignment
%
\STATE \hspp $\Theta_{k-1} \eqq  \{ \bth_{\ell,k-1}:   \ell \eqq 1, \dots, N_{k-1}\}$, cluster parameters
 %
 \STATE \hspp $D_{k-1}$, number of uniques cluster parameters 
\STATE \hspp $\Theta^*_{k-1} \eqq  \{ \bth^*_{\ell,k-1}: \ell\eqq 1, \ldots, D_{k-1} \}$, 
for unique clusters 
%
\STATE \vspace*{-.1in}
%
\STATE {\bf Transitioning from time} $(k-1)$ to $k$:
\STATE {\bf Input}: $\bX_{k-1}$, $\Theta^*_{k-1}$, 
 transition kernel $Q_{\bth}(\bx_{\ell, k-1}, \bx_{\ell, k})$ and probability of object staying in the scene $P_{k|k-1}$
%
 \STATE \vspace*{-.1in}
\IF{$\bx_{\ell,k-1}\in \bX_{k-1}$ leaves with probability $(1- P_{k|k-1})$} %\\
%${\bf{x_\ell(k-1)}} \in {\bf{X}_{k-1}} $ \texttt{leaves w.p. } }\\
\RETURN null
\ENDIF
\IF{$\bx_{\ell,k-1}\in \bX_{k-1}$  transitions with probability $P_{k|k-1}$} 
\STATE  $\bx_{\ell, k-1}  \sim Q_{\bth}( \bx_{\ell, k-1},  \bx_{\ell, k} )$ %\\
\RETURN $\bV^*_{k|k-1}$, survived states with  parameters 
\ENDIF
%
\STATE {\bf At time} $k$:
\FOR{ $i = 1 :  | \bV^*_{k|k-1}|$ }
\STATE Draw $\bth_{\ell, k}$  from $\zeta( \bth^*_{\ell, k-1},  \bth_{\ell,k})$ 
\ENDFOR
%\state 
\STATE Update $\Theta_k$ using \eqref{eq1}-\eqref{eq3} 
\STATE  Draw $\bx_{\ell,k} | \bth_{\ell,k}$  from \eqref{density} % \\
\RETURN $\{ \bx_{1, k},  \bx_{2, k}, \ldots \}$ and 
$\{ \bth_{1, k},  \bth_{2, k}, \ldots \}$
\end{comment}
\STATE {\textbf {At time}} $(k-1)$:\\
\hspace{0.25cm}\bll$\bX_{k-1}\eqq  \{ \bx_{\ell, k-1} \ \ldots \   \bx_{N_{k-1}, k-1} \}$: collection of object states vectors\\
\hspace{0.25cm}\bll $\mathcal{C}_{N_{k-1}}\eqq [ c_1, \ c_2, \ \ldots, \ c_{N_{k-1}} ]$, cluster assignment\\
\hspace{0.25cm}\bll $\Theta_{k-1} \eqq  \{ \bth_{\ell,k-1}:   \ell \eqq 1, \dots, N_{k-1}\}$, cluster parameters\\
 \hspace{0.25cm}\bll  $D_{k-1}$, number of uniques cluster parameters \\
\hspace{0.25cm}\bll  $\Theta^*_{k-1} \eqq  \{ \bth^*_{\ell,k-1}: \ell\eqq 1, \ldots, D_{k-1} \}$, 
for unique clusters \\
\vspace{0.1cm}
\STATE {\textbf {Transitioning from time} $(k-1)$ to $k$}:
\STATE \hspace{0.25cm}{\textbf {Input}}: $\bX_{k-1}$, $\Theta^*_{k-1}$, 
 transition kernel $\mathbb{Q}_{\bth_{\ell,k}}(\bx_{\ell, k-1}, \bx_{\ell, k})$ and probability of
 \STATE \hspace{1.7cm}object staying in the scene $P_{k|k-1}$
\IF{$\bx_{\ell,k-1}\in \bX_{k-1}$ leaves with probability $(1- P_{k|k-1})$} %\\
%${\bf{x_\ell(k-1)}} \in {\bf{X}_{k-1}} $ \texttt{leaves w.p. } }\\
\RETURN null
\ENDIF
\IF{$\bx_{\ell,k-1}\in \bX_{k-1}$  transitions with probability $P_{k|k-1}$} 
\STATE  $\bx_{\ell, k-1}  \sim \mathbb{Q}_{\bth_{\ell,k}}( \bx_{\ell, k-1},  \bx_{\ell, k} )$ %\\
\RETURN $D_{k|k-1}$: number of unique cluster, $\bV^*_{k|k-1} \in \mathbb{R}^{D_{k|k-1}}$: size vector, \\
\hspace{1.7cm} and $\Theta_{k|k-1}$: collection of survived parameters
\ENDIF
\vspace{0.1cm}
\STATE {\textbf {At time}} $k$:
\FOR{ $\ell = 1$ \TO $| \bV^*_{k|k-1}|$ }
\STATE Draw $\bth_{\ell, k}$  from $\zeta( \bth^*_{\ell, k-1},  \bth_{\ell,k})$ according to \ref{condd}
\STATE  Draw $\bx_{\ell,k} | \bth_{\ell,k}$  from \eqref{density}
\ENDFOR
%\STATE {\textbf{Update}} $\Theta_k$ using \eqref{eqpy1}-\eqref{eqpy3} 
%\vspace*{-.2cm}
 % \\
\RETURN $\{ \bx_{1, k},  \bx_{2, k}, \ldots \}$ and 
$\{ \bth_{1, k},  \bth_{2, k}, \ldots \}$
\end{algorithmic}
\end{algorithm}

\section{Learning Model}
\label{dpylearning}
The DPY-STP algorithm summarized in algorithm \ref{alg3} provides the density estimation of objects at time step $k$ in \ref{density}. The procedure then updates the estimated belief by using the set of measurements ${\mathcal{Z}_k }= \{{\bf{z}_{1,k},\dots, {\bf{z}_{M_k,k}}}\}$ received at time step $k$ to update the trajectory of objects. Based on theorem \ref{thm3} we may use a Dirichlet process mixture to update our model discussed in algorithm \ref{alg3}. This learning procedure is summarized in Algorithm \ref{alg4}. 

We assume that each measurement is associated only with one object. The measurements are also independent of each other. we can thus exploit Dirichlet process mixtures with the base distribution drawn from Algorithm \ref{alg3} to update our belief. However, the identity of the object that corresponds to a particular measurement is not known. As the objects are already labeled from the DPY clustering, the DPM model is used to learn the association between each measurement and its corresponding object. The clustering of the measurements first uses the DPY model result for the state distribution from Theorem \ref{thm4},
\begin{flalign}
\label{dpeq}
{\bf{x_{\ell,k}}} | {\bf{x}_{1,k}},\dots, {\bf{x}_{\ell-1,k}}, {\bf{X}}_{k|k-1}, \Theta_k \sim p({\bf{x}_{1,k}}|{\bf{x}_{1,k}},\dots, {\bf{x}_{\ell-1,k}},{\bf{X}}_{k|k-1}, \Theta^\star_{k|k-1}, \Theta_k)
\end{flalign}
and 
\begin{equation}
{\bf{z}_{l, k}} | {\bf{x}_{\ell,k}},{\bf{\theta^\star_{\ell, k}}} \sim R(\bz_{l,k}|{\bf{x}_{\ell, k}},{\bf{\theta^\star_{\ell, k}}})
\end{equation}
for some distribution $R$ that depends on the measurement likelihood function.

{{\bf{Remark:}}} Note that the DPY-STP algorithm is closely related to DDP-EEM algorithm introduced in section \ref{DDP} and thus both algorithms are well-defined. One can derive DDP-EMM model from the DPY-STP model by setting  $d = 0$. The discount parameter $d$ is used to control the number of clusters in the model. Intuitively speaking, on account of power law property of Pitman-Yor modeling, the higher the number of unique (non-empty) clusters is, the higher the probability of having even more unique clusters is. Also, intuitively speaking, we aim that clusters with small number of objects to have a lower probability of having new objects. Hence, the DPY-STP is more flexible and is a better match for the tracking problem with a time-varying number of objects. With a maximum number of $N_k$ objects at time step $k$, an object may stay in the scene from the previous time step, leave the scene, or enter the scene for the first time. Thus, the object state would benefit from a larger number of available clusters to ensure all dependencies are captured.

\begin{algorithm}[th]
\caption{Dirichlet process mixture model used to associate measurements with objects.}
\label{alg4}
\begin{algorithmic}
\STATE {\bf Input}: $\{ \bz_{1, k},   \ldots, \bz_{M_k, k}\}$, 
$\{ \bx_{1, k},  \bx_{2, k}, \ldots \}$, $\{ \bth_{1, k},  \bth_{2, k}, \ldots \}$ 
%\STATE \vspace*{-.3in}
\STATE { \bf At time} $k$:
\FOR{ $m = 1 : M_k$ }
 \STATE Draw $\bz_{m,k} | \bx_{\ell, k}, \bth_{\ell, k}$ from \eqref{dpeq} 
   \RETURN  $\cca_{N_k}$, cluster assignment at time $k$  
      \ENDFOR
      \STATE  {\bf Update}:  ${\cal C}{\cal A}_k =  {\cal C}{\cal A}_{k-1}  \cup \cca_{N_k}$ 
\RETURN Number of clusters $N_k$, ${\cal C}{\cal A}_k$ 
and posterior distribution of $\bz_{m, k} | \bx_{\ell, k}, \bth_{\ell, k}$, $m\eqq 1, \ldots, M_k$
\end{algorithmic}
\end{algorithm}
\subsection{Bayesian Inference: Gibbs Sampler}
\label{gibbs_dpy}
Exact posterior computation for DPY-STP algorithm is difficult when the number of parameters and observations are high. Nevertheless, we can make use of Gibbs sampling for inference in the DPY-STP where the conjugate priors are used. To this end, we make use of the auxiliary random variables to identify the cluster associations for the measurements. The resulting sampler allows model and measurement parallelization. Note that inference in DPY-STP model depends directly on number of the clusters and the number of measurements at each time step. Under the cluster assignments $\mathcal{CA}_k$, we introduce a cluster indicator $\mathcal{C}_k = \{c_{1,k},\dots, c_{N_k,k}\}$ at time k such that $c_{i,k} = c_{j,k}$ if and only if $\theta_{i,k} = \theta_{j,k}$ and $c_{i,k} = \ell$ if and only if $\theta_{i,k} = \theta^\star_{\ell, k}$ ( Note that $\theta^\star_{\cdot, k}$'s indicate the unique parameters at time $k$). Note that the cluster indicator $\mathcal{C}_k$ partitions the set of $\{1,\dots, N_{k}\}$. Since realization of the Pitman-Yor process is a discrete random measure with probability one, we can marginalize this process and derive the successive conditional Blackwell-MacQueen distribution: 
\begin{equation}
\label{cond_prior}
\theta_{\ell,k}|\Theta \sim \sum\limits_{\Theta_k - \{\theta_{\ell,k}\}} \Gamma^1_{j,k} \delta_{\theta}(\theta_{\ell,k}) + \sum\limits_{\substack{\theta \in \Theta^\star_{k|k-1}\setminus \Theta\\ \theta\neq \theta_{\ell,k}}}\Gamma^2_{j,k} \nu({\bf{\theta}^\star_{\ell, k-1}},{\bf{\theta_{\ell, k}}}) \delta_{\theta}({\bf{\theta_{\ell,k}}}) + \Gamma^3_k H(\theta_{\ell,k}).
\end{equation}

Assuming the base measure $H$ is nonatomic, the required conditional distribution to do local inference is derived by marginalizing over the mixing measures:
\begin{flalign}
\label{aux_gibbs}
p(c_{i,k} = \ell | \mathcal{C}_k\setminus\{c_{i,k}\}, \mathcal{Z}_k, \text{rest}) \propto 
\begin{cases}
\Gamma^{1,-i}_{\ell, k} R(\bz_{l,k}|\bx_{\ell,k}, \theta^\star_{\ell,k})& \text{for cluster $\ell$ that has been selected}\\
\Gamma^{2,-i}_{\ell, k} R(\bz_{l,k}|\bx_{\ell,k}, \theta^\star_{\ell,k})&\text{for cluster $\ell$ that has not yet been selected}\\
\Gamma^{3,-i}_k \int R(\bz_{l,k}|\bx_{\ell,k}, \theta) dH(\theta) &\text{new cluster is created}
\end{cases}
\end{flalign}
where $\Gamma^{j,-i}_{\ell, k}$ is the probability of selecting the $c_{t,k} = \ell$ where $ t \neq i$ given by 
\begin{flalign}
\Gamma^{1,-i}_{\ell, k} = \frac{ \left[\sum \limits_{j=1}^{D_{k-1}}\left[V^\star_{k|k-1}\right]_j \eta_{j, k \mid k-1}  \delta_j(c_{\ell, k})+ \left[{\bf{V}}_k\right]_{\ell} \right]_{-i}- d}{\left[\sum\limits_{t=1}^{\ell-1}\sum \limits_{j=1}^{D_{k-1}}\left[V^\star_{k|k-1}\right]_j \eta_{j, k \mid k-1}  \delta_j(c_{t, k}) + \sum\limits_{t=1}^{\ell-1}\left[{\bf{V}}_k\right]_{t} \right]_{-i}+ \alpha }\\
\Gamma^{2,-i}_{\ell, k} =  \frac{\left[ \sum \limits_{j=1}^{D_{k-1}}\left[V^\star_{k|k-1}\right]_j \eta_{j, k \mid k-1}  \delta_j(c_{\ell, k})\right]_{-i} - d}{\left[\sum\limits_{t=1}^{\ell-1}\sum \limits_{j=1}^{D_{k-1}}\left[V^\star_{k|k-1}\right]_j \eta_{j, k \mid k-1}  \delta_j(c_{t, k}) + \sum\limits_{t=1}^{\ell-1}\left[{\bf{V}}_k\right]_{t} \right]_{-i}+ \alpha }\\
\Gamma^{3,-i}_k = \frac{|D_k|_{-i}d+\alpha} {\left[\sum\limits_{t=1}^{\ell-1}\sum \limits_{j=1}^{D_{k-1}}\left[V^\star_{k|k-1}\right]_j \eta_{j, k \mid k-1}  \delta_j(c_{t, k}) + \sum\limits_{t=1}^{\ell-1}\left[{\bf{V}}_k\right]_{t} \right]_{-i}+ \alpha }
\end{flalign}

where $\left[\cdot\right]_{-i}$ indicates the total number of object parameters observed excluding the $i$th object and $|D_k|_{-i}$ is the total number of unique clusters created at time $k$ before $i$th object is observed and $R$ is the likelihood function. Equation \ref{aux_gibbs} is derived by multiplying the likelihood function by the conditional prior derived in equation \ref{cond_prior}. 

To completely specify the sampling procedure, we need to update $\Theta^\star_k = \{\theta^\star_{1,k}, \dots, \theta^\star_{D_{k},k}\}$. To do so, draw a new value for $\theta^\star_{\ell,k}$ from a distribution proportional to 
\begin{equation}
\prod\limits_{\{\bz_{l,k}: \theta_{l,k} = \theta^\star_{\ell,k}\}} R(\bz_{l,k}|\bx_{\ell,k}, \theta^\star_{\ell,k}) dH(\theta^\star_{\ell,k}).
\end{equation}

\section{Properties of DPY-STP Model}
\label{pyprop}
{\subsection{Posterior Distribution}}
As mentioned in section \ref{gibbs_dpy}, this method induces a partition over $\{1,2, \dots, N_k\}$, i.e., an unordered collection of nonempty subsets such that the set is the disjoint union of the subsets and each element of the set belongs only to one and only one subset, which is exchangeable. Let $\mathcal{C}_k = \{c_{1,k}, \dots, c_{D_k, k}\}$ and $|\mathcal{C}_{k}| = \{\left[V_k\right]_1, \dots, \left[V_k\right]_{D_k}\} $ be the unordered collection of cluster (partition) assignment and its cardinality such that $|c_{j,k}| = \left[V_k\right]_j$ and $\sum\limits_{j=1}^{D_k} \left[V_k\right]_j = N_k$ at time $k$, respectively. Define $ \big(\left[V_k\right]^\ast_1,\dots,\left[V_k\right]^\ast_{D_k}\big)$ to be the size of ordered clusters (partition) such that $\left[V_k\right]^\ast_1\leq\dots,\leq \left[V_k\right]^\ast_{D_k}$. Due to exchangeability of the sequence associated with the cluster (partition) assignments, it is shown that exchangeable partition probability function (EPPF) is given in \cite{pitman1997} by 
\begin{equation}
\label{EPPF}
p(\left[V_k\right]^\ast_1, \dots, \left[V_k\right]^\ast_{D_k}) = \frac{\prod\limits_{j = 1}^{D_k} (\alpha + jd)}{\alpha^{\left[N_k\right]}} \prod\limits_{i = 1}^{D_k} (1-d)^{\left[V_k\right]^\ast_i}
\end{equation}
where $\alpha^{\left[n\right]} = \alpha (\alpha+1) \dots (\alpha+n-1)$. Note that if we set $d = 0$ the equation \ref{EPPF} reduces to the EPPF for the Dirichlet process with concentration parameter $\alpha$ in equation \ref{EPPFDir}. Note that the induced random partition by $\mathcal{C}_k$ at each time $k$ is distributed according to the equation \ref{EPPF}.

Furthermore, it is shown in \cite{pitman1997} if the distribution on the cluster parameters drawn from $\text{DPY-STP}_k| \text{DPY-STP}_{k-1}$ in Theorem \ref{condd} and $d>0$. then posterior distribution given $\theta^\star_{1,k}, \dots, \theta^\star_{D_k,k}$ is the distribution of the random measure 
\begin{equation}
\label{posterior}
B_n\sum\limits_{i=1}^{D_k} \pi_i \delta_{\theta^\star_{i,k}} + (1 - B_n) \tilde{H}
\end{equation}
where $B_n \sim Beta (N_k - D_k d, \alpha + D_k d)$, $(\pi_1, \dots, \pi_{D_k}) \sim \text{Dirichlet}(\left[V_k\right]_1-d, \dots, \left[V_k\right]_{D_k}-d)$, and $\tilde{H} \sim \mathcal{PY}(d, \alpha+ D_k d, G^\ast)$ where $G^\ast = \sum_{\Theta_k} \Gamma^1_{j,k} \delta_{\theta_{\ell,k}} + \sum_{\Theta^\star_{k|k-1}\setminus \Theta_k}\Gamma^2_{j,k} \zeta({\bf{\theta}^\star_{\ell_k-1}},{\bf{\theta_{\ell, k}}}) \delta_{{\bf{\theta_{\ell, k}}}} + \Gamma^3_k H$. Note that $B_n$ and $(\pi_1, \dots, \pi_{D_k})$, and $\tilde{H}$ are mutually independent. 

{\subsection{Posterior Consistency of DPY-STP model}}

As discussed in section \ref{ddpconsis}, DDP-Priors are consistent for estimating the distributions. The statistical model introduced in this paper along with all the statistical models based on Pitman-Yor may be used to estimate the distributions and track the objects. However, the Pitman-Yor process prior assumes the inconsistency of the Gibbs processes priors to estimate the distributions. it is shown the conditions under which Gibbs processes are consistent (Section 3, Theorem 1of \cite{de2011}). Consistency of Pitman-Yor processes is the direct result of Gibbs prior consistency. The following proposition summarizes these conditions:  
\begin{prop}
\label{conspy}
Let $G_k\sim \mathcal{PY}(d,\alpha, H)$ be the prior distribution drawn from a Pitman-Yor Process. The posterior distribution of $G_k|\mathcal{Z}_k$ is consistent at probability measure $G_0$ if and only if one the following conditions holds:
\begin{enumerate}
\item $G$ is the mixture of at most $\lceil{|\frac{\alpha}{d}|}\rceil$ degenerated measures, i.e., $G_0$ is discrete
\item $H$ is proportional to  $G_{0,c}$ where $G_{0,c}$ is continuous part of the probability measure $G_0$
\item $d = 0$, which is equivalent to the consistency of the Dirichlet process.
\end{enumerate}
\end{prop}
\begin{proof}
This Proposition immediately results from the Gibbs prior consistency theorem \cite{de2011}.
\end{proof}
\section{Simulation Results}
\label{sim}
\subsection{Comparison to Multi-Bernoulli Filtering}
The performance of the DDP-EEM model
 is demonstrated and compared to the labeled multi-Bernoulli filter (LMB)
 for a radar target tracking  simulation example. The  time-dependent number of 
 targets are assumed to move  according to the coordinated turn motion model,
 and there is a maximum of ten targets.
 Note that this same example is used for the LMB in \cite{Vo2014}.
 The unknown state parameters 
of the $\ell$th target at time $k$ are 
 the  Cartesian coordinates of  the two-dimensional (2-D) position 
 $[x_{\ell, k} \  y_{\ell, k}]^T$,
 target velocity $[\dot{x}_{\ell, k} \ \dot{y}_{\ell, k}]^{T}$ and 
 target turn rate $\omega_{\ell, k}$.
 The $\ell$th state  vector 
 is given by  $\bx_{\ell, k} \eqq [x_{\ell, k} \  y_{\ell, k} \  \dot{x}_{\ell, k}\  
 \dot{y}_{\ell, k} \  \omega_{\ell, k}]^T$, $\ell\eqq 1, \ldots, N_k$, where
 $N_k$ is the time-dependent target cardinality.
The actual time-dependent trajectories  are shown in Figure \ref{act_traj}.
The transition probability density $p(\bx_k$$\mid$$\bx_{k-1})$
for the coordinated turn motion model is assumed 
 to be a Gaussian distribution with  mean  vector 
$\bmu \eqq [ \bze^T \ \omega_{k-1}]^T$
where $ \bze \eqq A_{\omega_{k-1}}  \bx_{k-1}$ and 
covariance matrix $Q\eqq \text{diag}( [  \sigma^2_w B B^T , \sigma^2_u])$
where $\sigma_w\eqq15$ m/s$^2$,  $\sigma_u \eqq \pi/180$ radians/s, and 
\ben
A_{\omega_{k-1}}\!\!=\!\!
\begin{bmatrix}
1 &\frac{\sin(\omega_{k-1})}{\omega_{k-1}}&0&-\frac{1-\cos(\omega_{k-1})}{\omega_{k-1}}\\
0&\cos(\omega_{k-1})&0&-\sin(\omega_{k-1})\\
0 & \frac{1-\cos(\omega_{k-1})}{\omega_{k-1}} & 1 & \frac{\sin(\omega_{k-1})}{\omega_{k-1}}\\
0 & \sin(\omega_{k-1}) & 0 & \cos(\omega_{k-1})
\end{bmatrix}\!\!,
B = \begin{bmatrix}
\frac{1}{2} & 0 \\
1 & 0\\
0& \frac{1}{2}\\
0&1
\end{bmatrix}\!.
\een
\begin{figure*}[ht]
\centering
\subfloat[ ]{ \label{act_traj}
 \resizebox{3in}{!}{\includegraphics{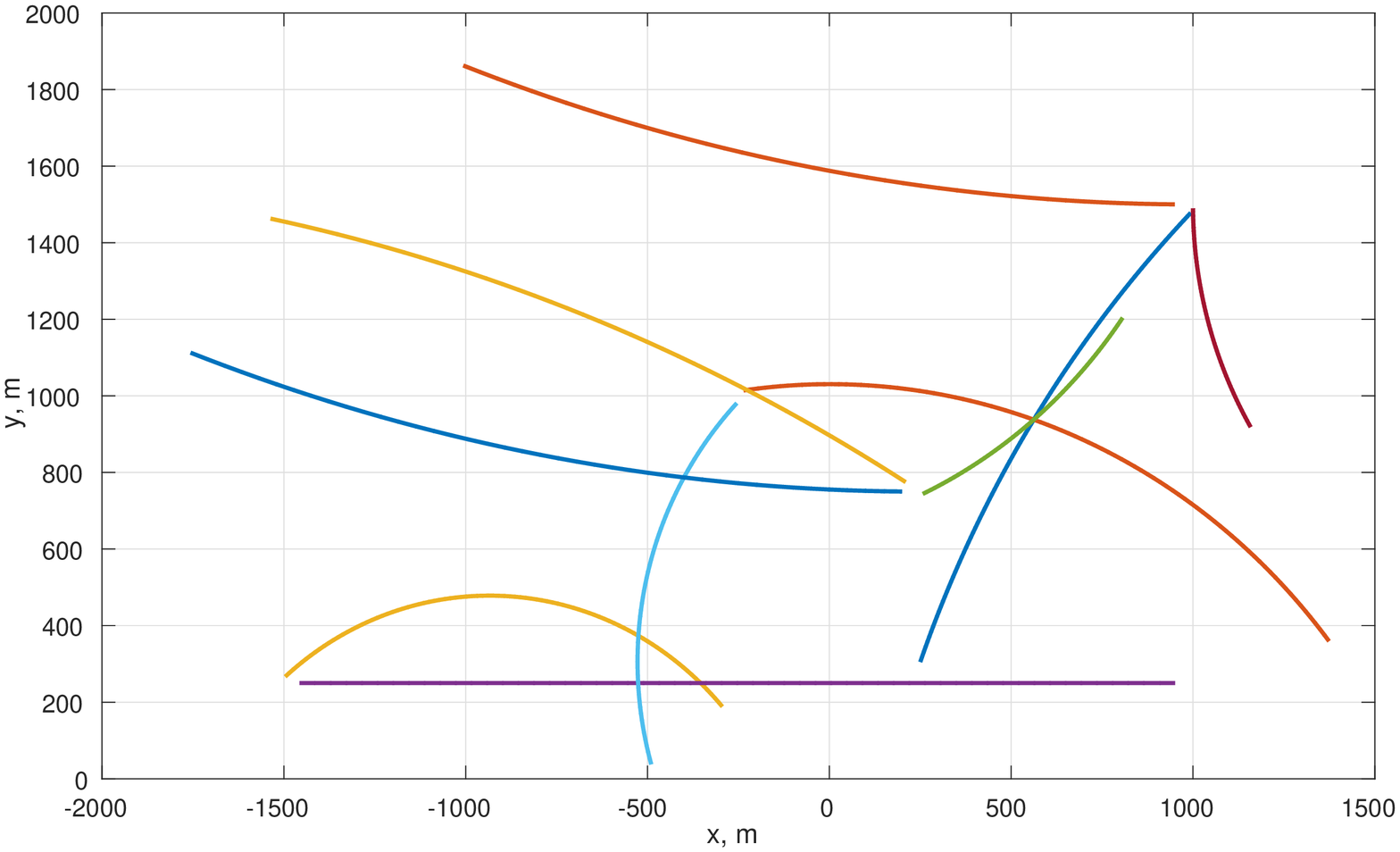}}}
\hspace*{-.2in}
\subfloat[ ]{ \label{fig:fig3}
 \resizebox{4in}{!}{\includegraphics{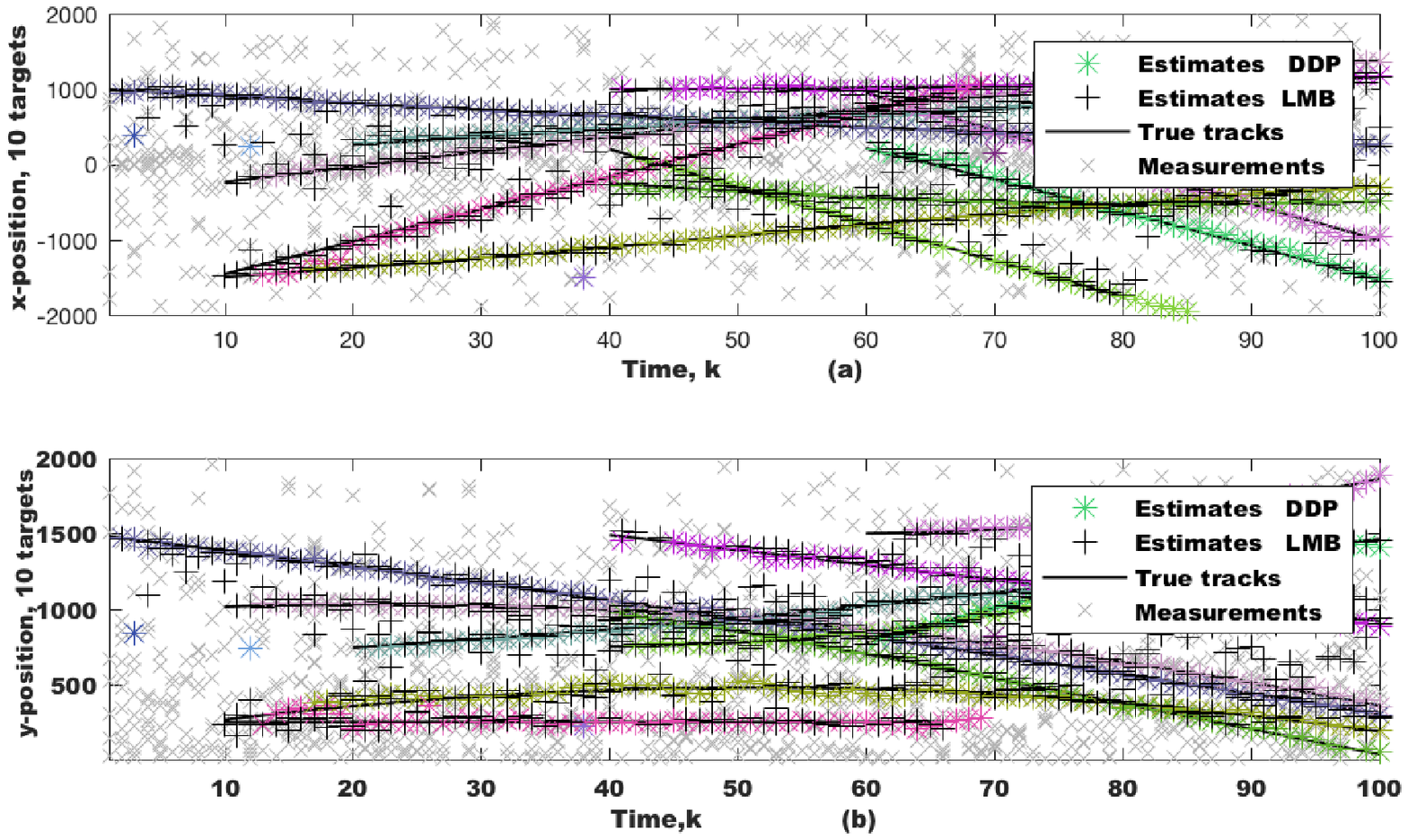}}}
 \vspace*{-3mm}
\caption{ (a) Actual target trajectories. (b) 
Actual and estimated $x$ (top) and  $y$ (bottom) position 
versus time $k$ using DDP-EMM and LMB methods.}
\vspace*{-.1in}
\end{figure*}
\begin{figure*}[t]
\centering
\subfloat[ ]{ \label{fig:fig1}
\resizebox{3.5in}{!}{\includegraphics{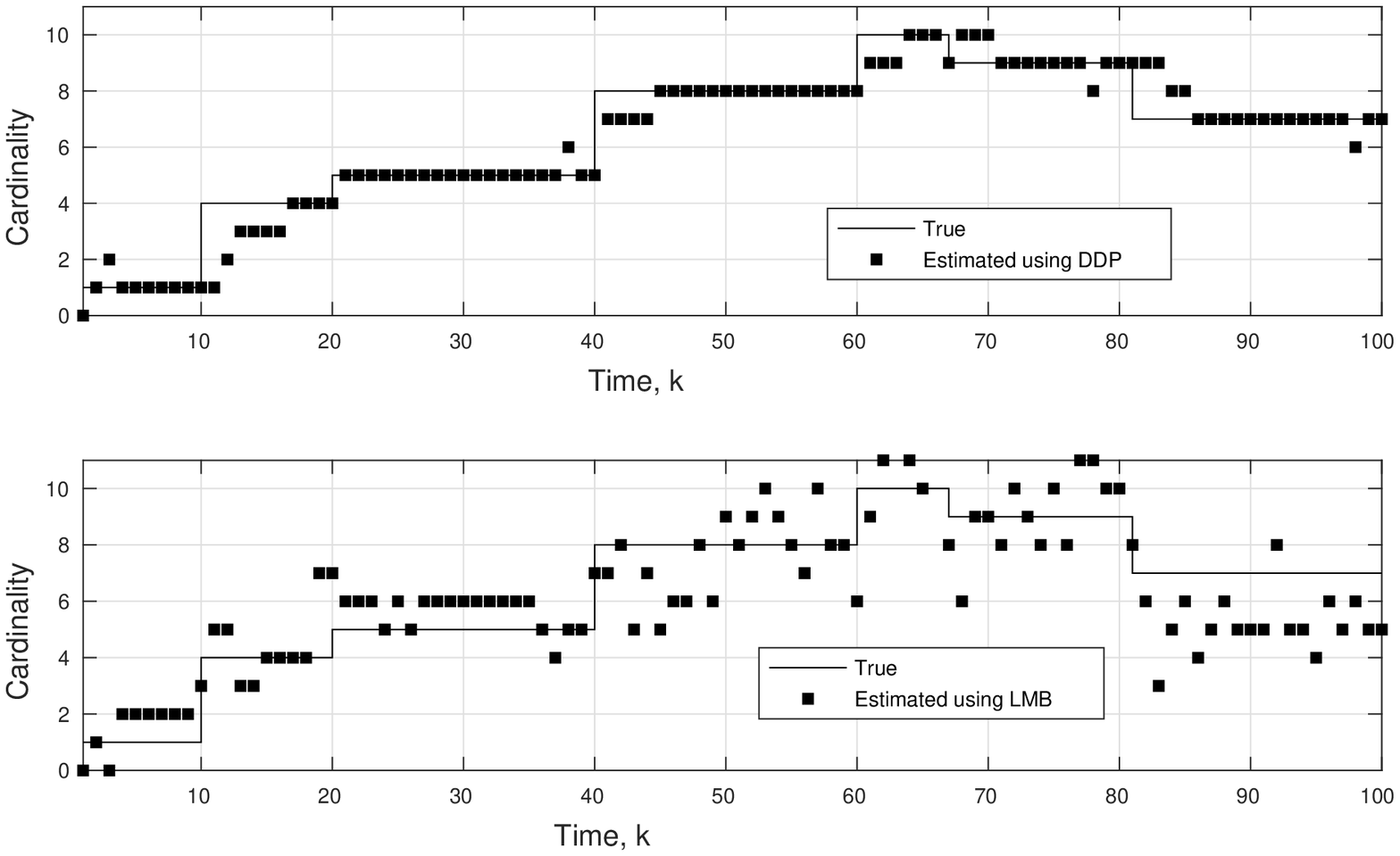}}} 
\hspace*{-.37in}
\subfloat[ ]{ \label{fig:fig2}
\resizebox{3.5in}{!}{\includegraphics{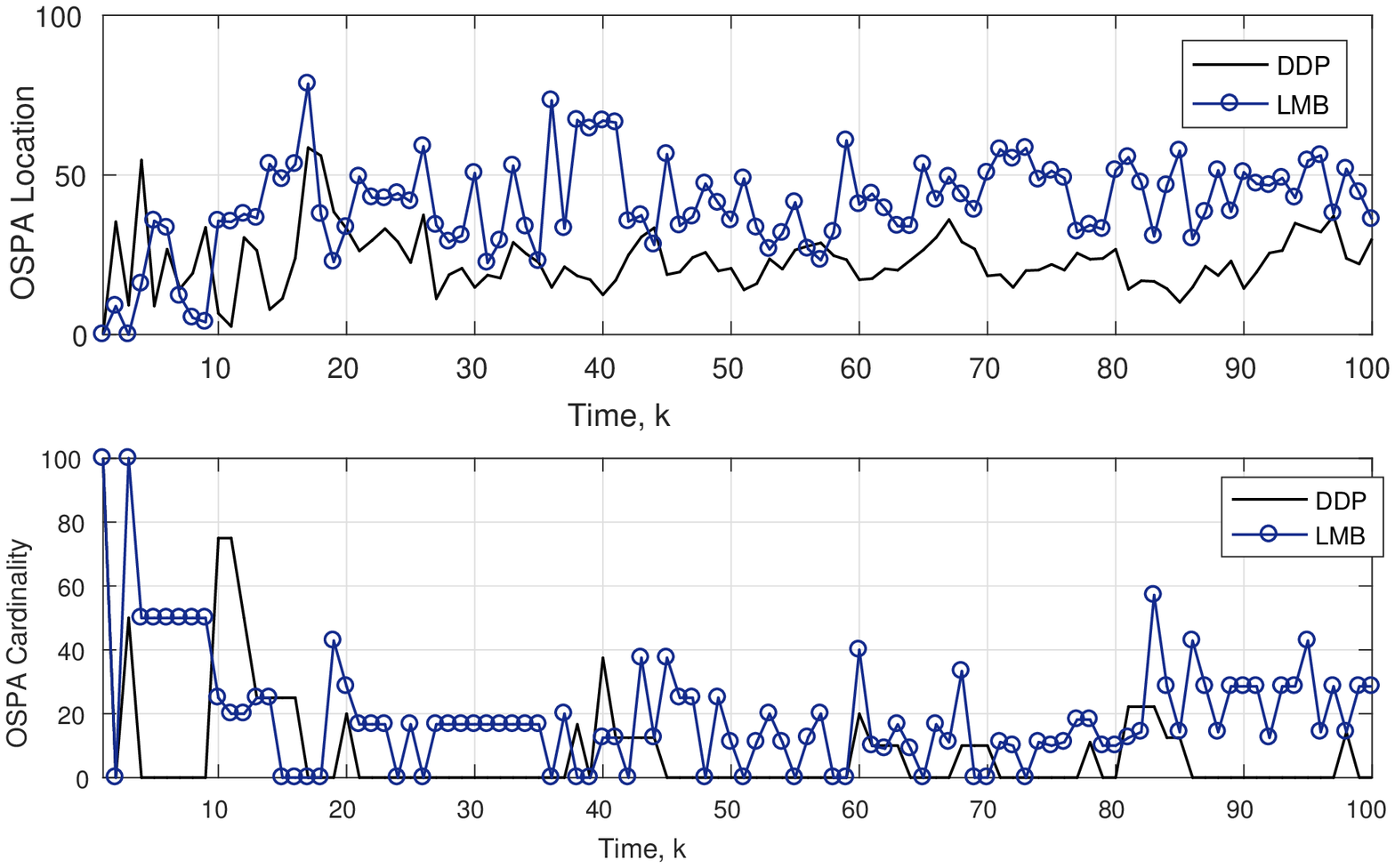}}}
\caption{(a) Comparison between cardinality estimation for DDP (top) and LMB (bottom)
when tracking 10 objects. (b) OSPA location (top) and cardinality (bottom)
of order $p \eqq 1$ and cut-off $c\eqq100$. }
\end{figure*}

We select the probability of a target remaining at a scene during
transitioning to be  $\text{P}_{\ell, k \mid k-1}\eqq 0.95$, for all $\ell$.
The times each target enters and leaves the scene are summarized in
Table \ref{existence}. 
The measurement vector $\bz_k \eqq [\phi_k \  r_k]^{T}$ at time $k$ 
includes  bearing $\phi_k$ and range $r_k$, where 
$r$$\in$$[0, 2,000]$ m and $\phi$$\in$$[-\frac{\pi}{2},\frac{\pi}{2}]$.
The measurement noise is assumed zero-mean  Gaussian 
with variance $\sigma^2_{r}\eqq25$ and $\sigma^2_{\phi}\eqq ({\frac{\pi}{180}})^2$. For the simulations, 10,000 Monte Carlo runs were used, The overall observed time steps is considered to be $K = 100$ and the signal-to-noise-ratio (SNR) is -3 dB. In our proposed model, we used a normal-inverse Wishart distribution, $\mathcal{NIW}(\mu_0 , \lambda, \nu, \Psi)$, with values $\mu_0= 0.001, \lambda = 0, \nu = 50$, and an identity matrix for $\Psi$ as prior on the space of parameters. We consider a Gamma distribution as prior on the concentration parameter $\alpha$, $\Gamma(\alpha; 1,0.1)$
Using the proposed DDP-EMM and inferential methods the estimated $x$ and $y$ coordinates are shown to match the true coordinates in Figures 1(a) and 1(b), respectively. When compared to the LMB in Figure 2, the DPY-EM shows a higher estimation accuracy for the x and y coordinates in Figure 2a. The increase in performance is also demonstrated consistently using the OSPA measurement, both for the range and the time-dependent object cardinality in Figure 2b.

Figure \ref{fig:fig3} displays the actual and estimated target  trajectories  
 for the proposed DDP-EMM and (LMB) methods in 
 10,000 Monte Carlo (MC) runs.
As shown in Figure \ref{fig:fig1}, the DDP-EMM has higher accuracy than the 
LMB  when estimating the time-dependent target cardinality.
This is also demonstrated using the optimal sub-pattern assignment (OSPA) 
metric (of order $p\!=\!1$  and cut-off $c\!=\!100$) in Figure \ref{fig:fig2}.
The OSPA location for both methods is  compared in Figure \ref{fig:fig2} (top). 
Note that the lower the OSPA metric, the higher the corresponding performance.
  We observe that the DDP-EEM  method often performs
better than the LMB; this may be due to the fact that the LMB requires 
  approximations  when updating the target state estimates.
\begin{table}[ht]
\centering
\caption{Target existence over time.}
\label{existence}
\begin{tabular}{|l|l|l|l|}  \hline
Object& Presence   & Object & Presence  \\ \hline
 Object 1&$0\leq k\leq100$ & Object 6 & $40\leq k\leq100$   \\
 Object 2 & $10\leq k\leq100$  & Object 7  & $40\leq k\leq100$   \\
Object 3 & $10\leq k\leq100$  & Object 8 & $40\leq k\leq80$  \\
 Object 4 & $10\leq k\leq60$  & Object 9& $60\leq k\leq100$  \\
Object  5 & $20\leq k\leq80$  & Object 10 & $60\leq k\leq100$  \\ \hline
\end{tabular}
\end{table}

The  DDP-EMM and  LMB can both 
 track the targets.  However, the DDP-EMM  is 
computationally more efficient  and has a higher 
 tracking performance.
As shown in  
 Figure \ref{fig:fig1},  the LMB drastically overestimates the cardinality
 of the 10 targets, when compared to the DDP-EMM,  
 showing the elimination of the posterior cardinality bias.
  This is because the LMB
 is highly sensitive to the presence of clutter.
 The OSPA location  and OSPA cardinality measures of both methods
are compared in Figure \ref{fig:fig2}.
We observe that the DDP-EEM  method often performs better than 
the LMB due to approximations assumed in the LMB filtering to update the tracks. 

\subsection{DDP-EMM and Low SNR: Moving Cars with Turn}
\label{cars}
In this section, we show that DDP-EEM algorithm may track objects in the presence of high noise through simulations. We consider five moving cars where it is  assumed that each car may enter, leave, or turn at any time. Each car comes to scene at different time and must follow the cars in front of it. The goal is to estimate the location/range of each car as well as the number of the cars in the scene at each time step based on the noisy measurements received from the sensor. 

\begin{figure*}[ht]
\centering
\hspace*{-.4in}
\subfloat[ ]{ \label{track_car}
 \resizebox{3in}{!}{\includegraphics{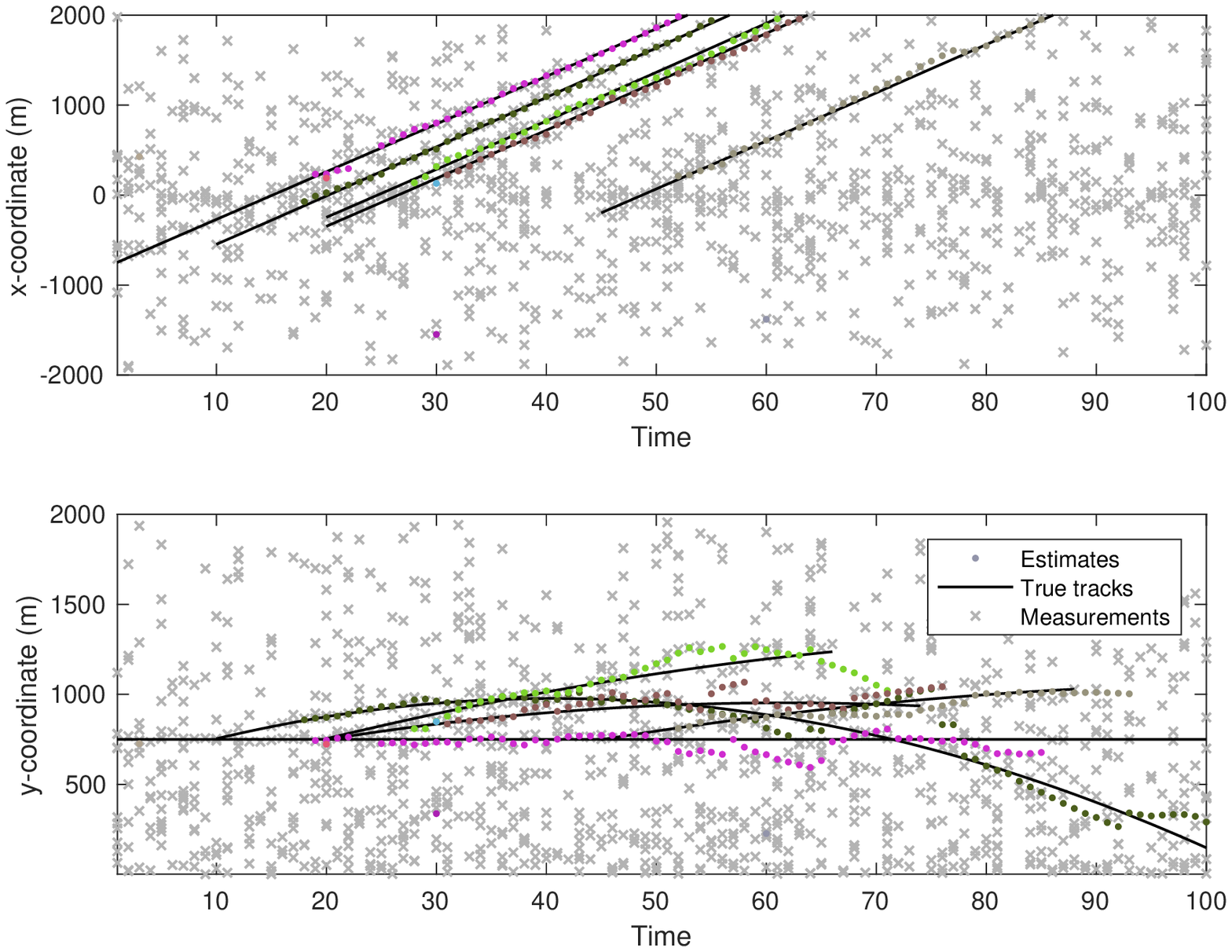}}}
\subfloat[ ]{ \label{loc}
 \resizebox{3in}{!}{\includegraphics{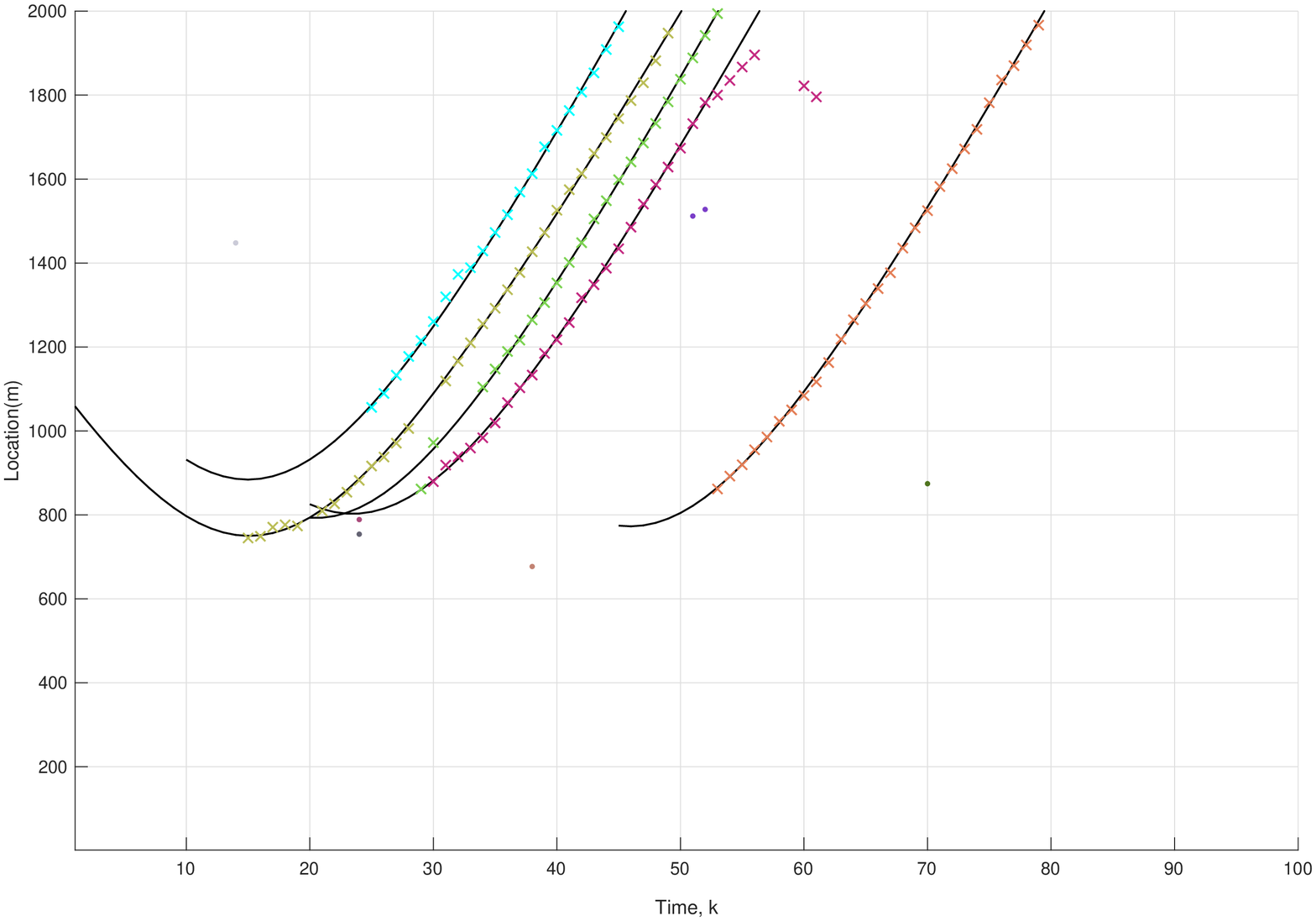}}}
\vspace*{-1mm}
\caption{  (a) $x$-coordinate and $y$-coordinate estimation using DDP-EMM model. (b) Location estimation using DDP-EMM.}
\vspace*{-.1in}
\end{figure*}

\begin{figure*}[htb]
\centering
\subfloat[ ]{ \label{card}
 \resizebox{3.5in}{!}{\includegraphics{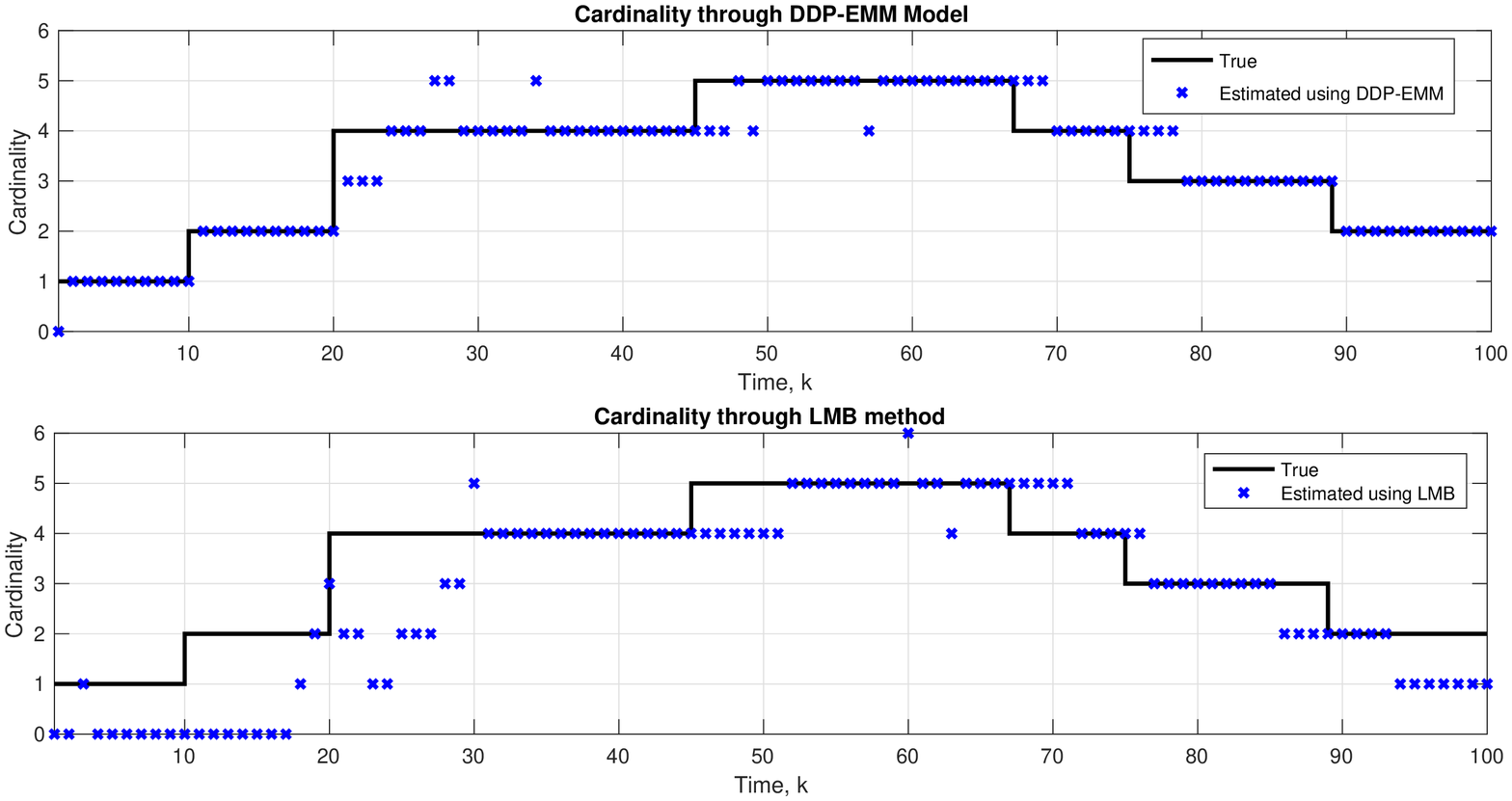}}}
\hspace*{-.4in}
\subfloat[ ]{ \label{ospa_car}
 \resizebox{3.5in}{!}{\includegraphics{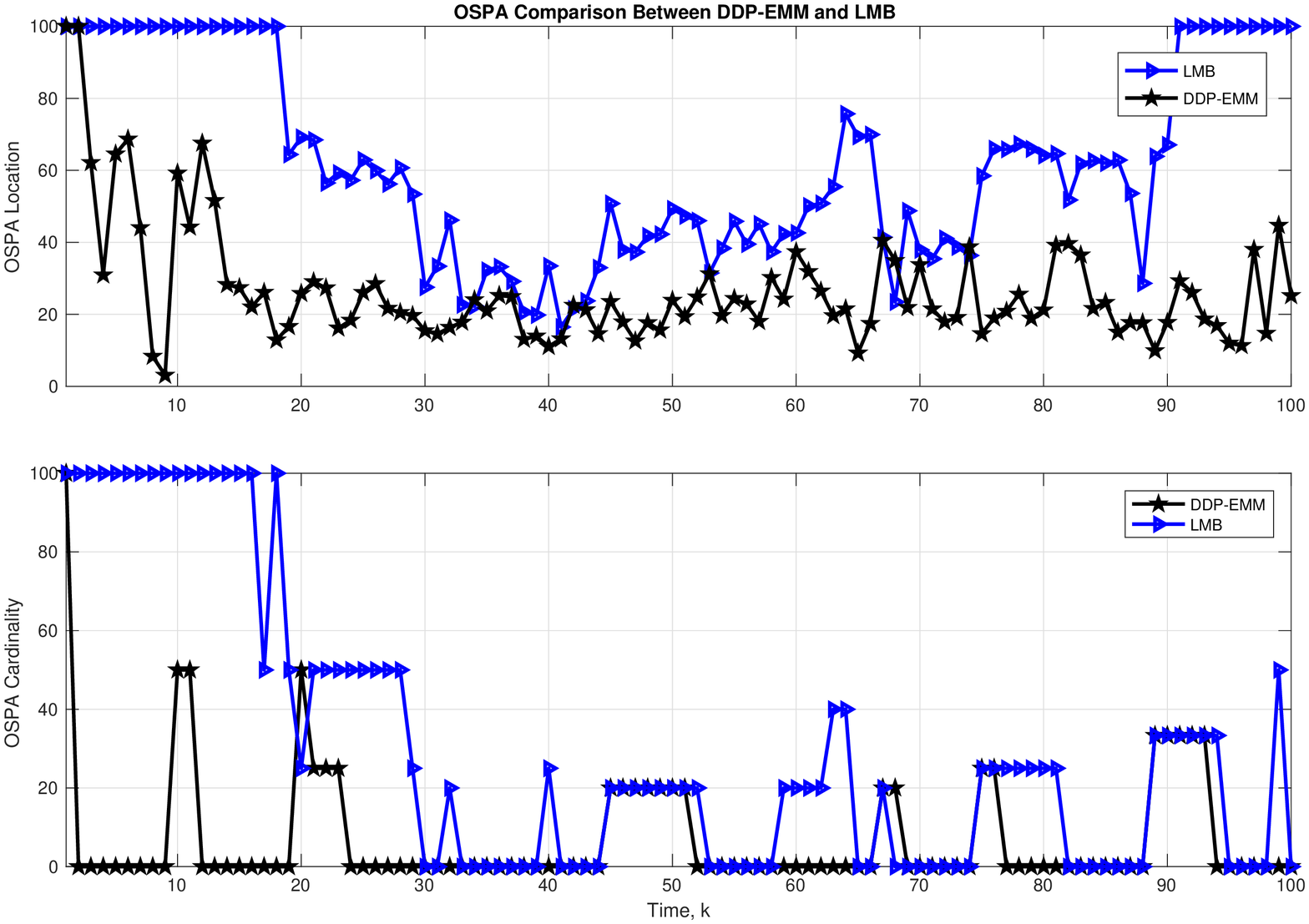}}}
\vspace*{-1mm}
\caption{  (a) Cardinality estimation using DDP-EEM and LMB. (b) OSPA comparison between DDP-EMM and LMB for cut-off $c=100$ and order $p=1$.}
\vspace*{-.1in}
\end{figure*}
The unknown state of each car is considered to be $[x, y, \dot{x}, \dot{y}, \omega]^T$ where $(x , y)$, $(\dot{x}, \dot{y})$, and $\omega$ are the location, velocity, and turning rate, respectively. The sensor only collects information about the range and angle at each time step. An additive Gaussian noise is assumed throughout simulations. The SNR for this model is $-3$ dB. In this scenario, the objects are assumed to be located near one another which makes the model complicated to analyze.  We compare the tracker introduced in this paper to the LMB tracker. We illustrate through simulations that DDP-EEM algorithm produces an accurate estimate of the location and cardinality despite high noise. We assume We assume a normal-inverse Wishart distribution, $\mathcal{NIW}(\mu_0 , \lambda, \nu, \Psi)$, with values $\mu_0= 0.01, \lambda = 0, \nu = 100$, and an identity matrix for $\Psi$ as prior on the space of parameters. We consider a Gamma distribution as prior on the concentration parameter $\alpha$, $\Gamma(\alpha; 1,0.3)$. Running 10,000 Monte Carlo (MC) simulations, the estimated cardinality and the OSPA metric for the location estimation error is depicted in Fig. \ref{card} and Fig. \ref{ospa_car}, respectively. For OSPA metric we consider the order $p =1$ and the cut-off $c = 100$. Figure \ref{track_car}, Figure \ref{loc} displays the $x$-coordinate and $y$-coordinate estimation and location of the objects using the DDP-EMM tracker, respectively. As shown in Fig. \ref{card} and Fig. \ref{ospa_car}, under same conditions, if the objects are located close to each other, the proposed DDP-EMM algorithm outperforms the LMB method and estimates the location more accurately.

%Different scenarios may generate different level of noise depending on the scene. We show the estimation error under different noise level. We demonstrate through simulations that the proposed DDP-EMM model not only has a robust performance in the presence of high level of noise but also performs well when the objects are closely located to one another.  

%Figures \ref{fig:loc} and \ref{fig:cardi} display the actual and estimated range of five moving cars and the estimated cardinality at each time using DDP-EEM algorithm using 10,000 Monte Carlo (MC) simulations. 

\subsection{DDP-EMM under Different SNR Values }
\begin{figure}[ht]
\centering
\includegraphics[width=12cm]{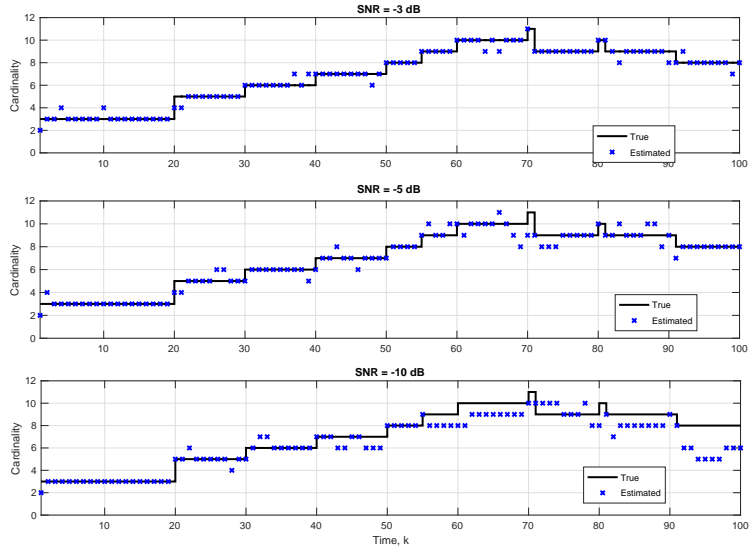}
\caption{ Cardinality estimation in the presence of different SNR values.}
\label{c-snr}
\end{figure}
We assume the same scenario as discussed in the section \ref{cars}. However, in this setup cars, we assume there is no turn, meaning $\omega = 0$. The unknown state of  $[x, y, \dot{x}, \dot{y}]^T$ and the measurements contain the range. We put our proposed DDP-EMM method to the test under different SNR values. With the DDP-EMM setup, we model the state parameters as a realization of the proposed process. We assume Gaussian distributions throughout this simulation. If we learn the states with mean of zero, our model reduces to that of constant acceleration model and by assuming a non-zero mean we may consider the fast changes. We simulate the algorithms for SNR = $-3$ dB, $-5$ dB, and $-10$ dB. 
\begin{figure}[h]
\centering
\includegraphics[width=12cm]{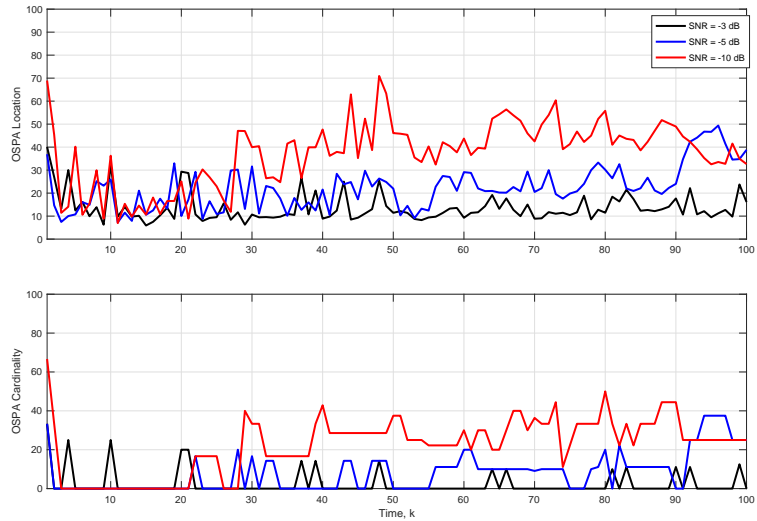}
\caption{DDP-EMM performance for SNR = $-3$ dB, SNR = $-5$ dB, and SNR = $-10$ dB.}
\label{o-snr}
\end{figure}
Place a normal-inverse Wishart distribution, $\mathcal{NIW}(\mu_0 , \lambda, \nu, \Psi)$, with values $\mu_0= 0, \lambda = 0, \nu = 100$, and an identity matrix for $\Psi$ as prior on the space of parameters and a Gamma distribution as prior over the concentration parameter $\alpha$, $\Gamma(\alpha; 1,0.2)$; running 10,000 Monte Carlo (MC) simulations results; Figure \ref{c-snr} presents the cardinality of the model under various SNR values. As shown in this figure, this method works perfectly for high SNR values and even though the SNR is very high, we are still able to obtain the correct cardinality of the states most of times. 

Figures \ref{o-snr} depicts the performance of this method under different SNR values. Note that for high SNR values the OSPA metric is still fairy low which verifies the good performance of this method. 

\subsection{DPY-STP method}
\label{py1}

\begin{table}[htb]
\centering
\caption{Time intervals that objects enter/leave the scene}
\label{existence}
\begin{tabular}{| c | c | c | }   \hline
Object  & Time  step entering scene & Time step leaving 
the scene \\  \hline
\hspace*{2mm} & & \\[-0.5mm]
 Object 1  &  $k\eqq0$  & $k \eqq70$   \\[2mm]
Object 2  & $k \eqq 5$
 & $k \eqq 100$ \\[2mm]
Object 3  & $k \eqq10$
 & $k \eqq 100$  \\[2mm]
Object 4  & $k \eqq20$ & $k \eqq 45$ \\[2mm]
Object 5 &$k \eqq 30$ & $k \eqq 80$ \\ \hline
\end{tabular}
\end{table}

The DPY-EM multiple object tracking method
is implemented using MCMC sampling methods,  together with 
Algorithms \ref{alg1} and \ref{alg2}.  
To demonstrate the performance of this method, 
we simulated a dynamic linear tracking 
example using five objects that enter and leave the 
scene at different times, as summarized in  Table \ref{existence}. 
The performance is compared to that of  the labeled multi-Bernoulli 
(LMB) approach. 
\begin{figure*}[!tbp]
 \begin{center}
 \resizebox{4in}{!}{\includegraphics{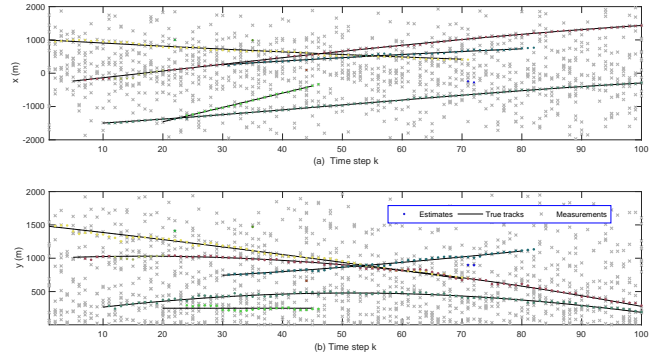}}
 \end{center}
\caption{True and estimated 
(a) $x$-coordinate and (b) $y$-coordinate
as a function of the time step $k$ for five objects.}
 \label{fig:coordinates}
\end{figure*}

\begin{figure*}[!tbp]
\begin{center}
\subfloat[ ]{ \label{fig:card}
    \resizebox{4in}{!}{\includegraphics{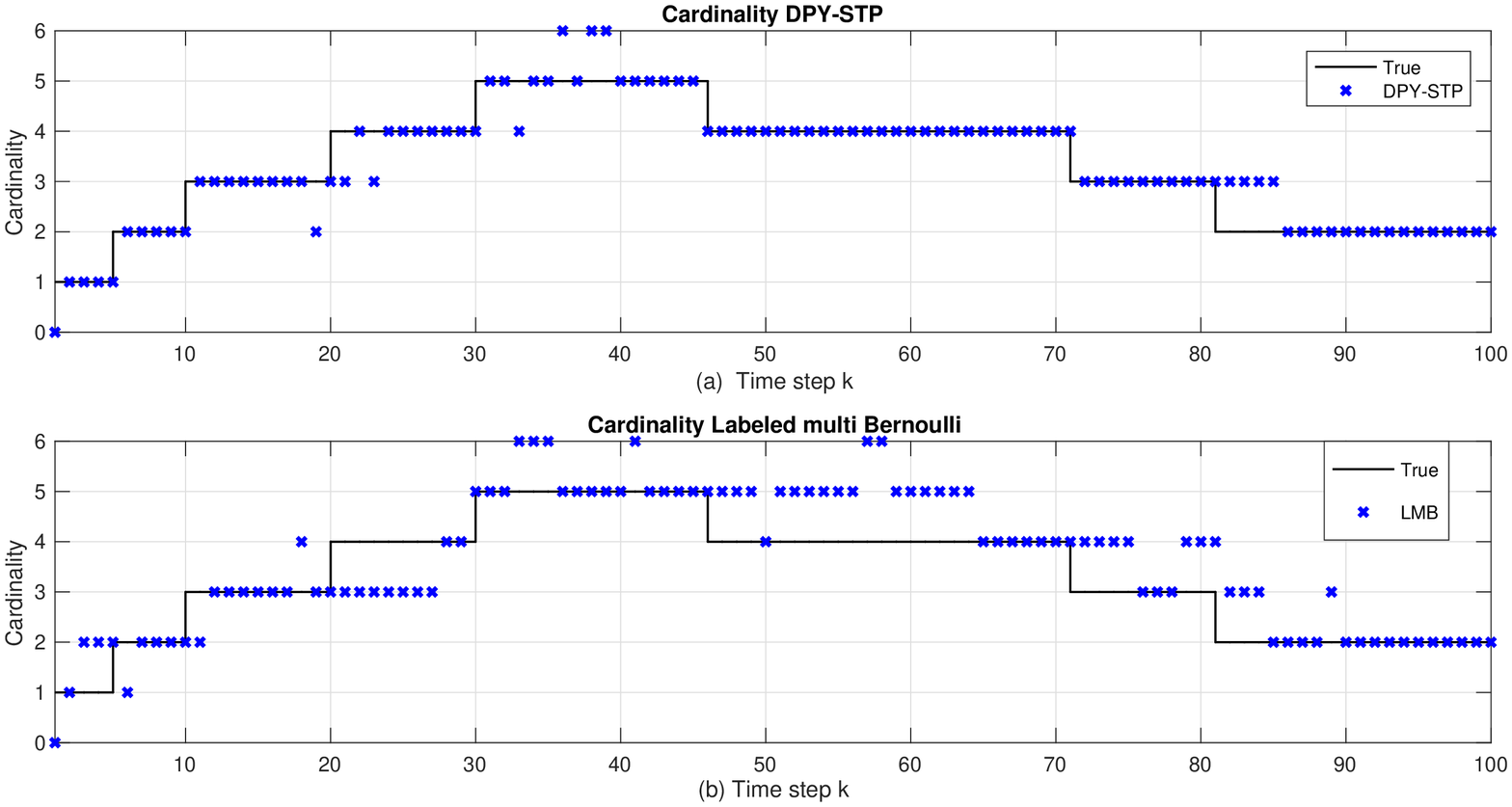}}}   \\
\subfloat[ ]{ \label{fig:ospa}
  \resizebox{4in}{!}{\includegraphics{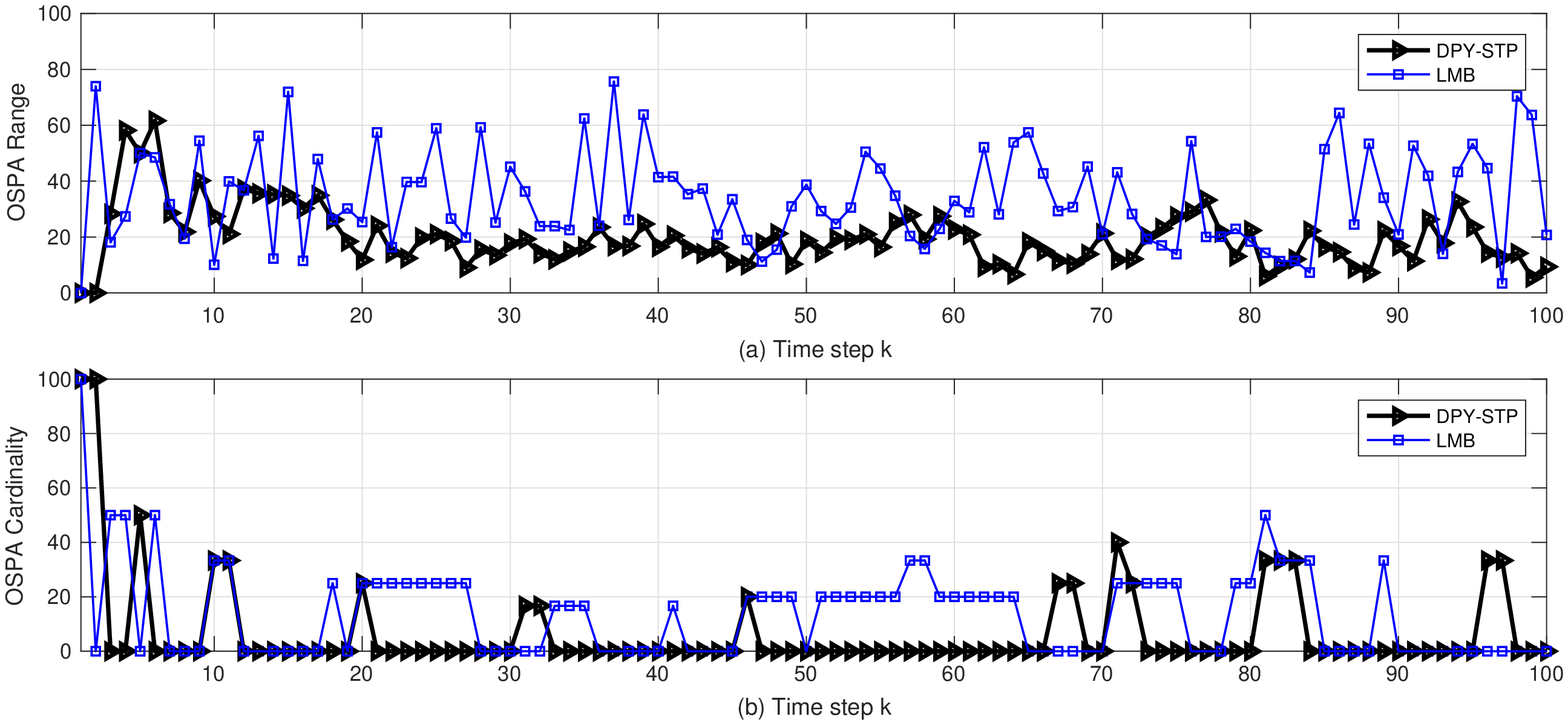}}}
  \end{center}
  \caption{(a)  True and estimated $x$-coordinate (top) and $y$-coordinate
  (bottom) as a function of time step $k$  for 5 objects. (b)
  OSPA (order $p\eqq1$ and cut-off $c \eqq100$
   for range (top) and (b) cardinality (bottom) averaged over 10,000 MC simulations
    for the DPY-STP  and the labeled multi-Bernouli (LMB) based tracking approaches.}
    \label{comb}
\end{figure*}

For the simulations, 10,000 Monte Carlo runs were used, 
The overall observed time steps is assumed to be $K\eqq100$ 
and the SNR is $-3$ dB.  
Also, 10,000 Monte Carlo runs were used in the simulations.
 The DPY-EM estimated $x$ and $y$ coordinates  are  
shown to match the true coordinates
 in Figures \ref{fig:coordinates}(a) and \ref{fig:coordinates}(b), respectively.
When compared to the LMB in Figure \ref{comb},
the DPY-EM  shows a higher  estimation accuracy
for the $x$ and $y$ coordinates in  Figure  \ref{fig:card}.
 The increase in performance is also demonstrated 
 consistently using the OSPA measurement,
both for the  range and the time-dependent object cardinality in 
Figure \ref{fig:ospa}.

\subsection{Comparison between DPY-STP and DDP-EMM}

Due to the flexibility of Pitman-Yor process and the fact that the object state benefits from a larger number of available clusters to ensure all dependencies are captured, we are expecting to obtain better results using DPY-STP, given the condition in Theorem \ref{conspy}. In this section, we compare both proposed methods and verify that the algorithm based on the dependent Pitman-Yor process may have better results than DDP-EMM. To do this end, we consider the problem of tracking 10 objects using both methods. We assume the base distribution to have a normal-inverse Wishart  distribution, $\mathcal{NIW}(\mu_0 , \lambda, \nu, \Psi)$ where $m_0 = 0, \lambda = 0, \nu = 100, \text{and } \Psi = I$. We select $\alpha$ and $d$ the same way as \ref{py1}.
\begin{figure*}[htb]
\centering
\subfloat[ ]{ \label{xypy}
 \resizebox{3.5in}{!}{\includegraphics{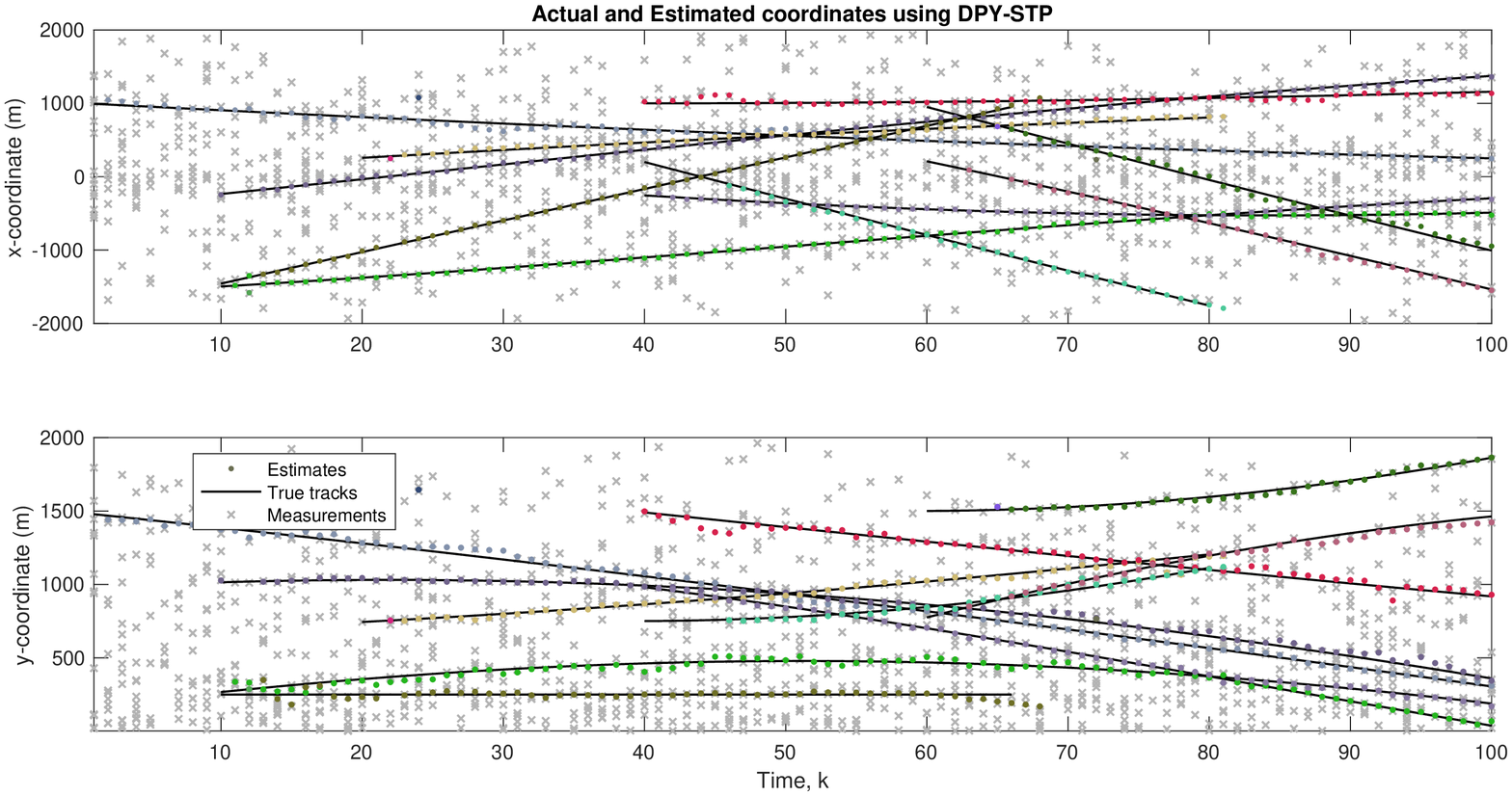}}}
\hspace*{-.4in}
\subfloat[ ]{ \label{xydp}
 \resizebox{3.5in}{!}{\includegraphics{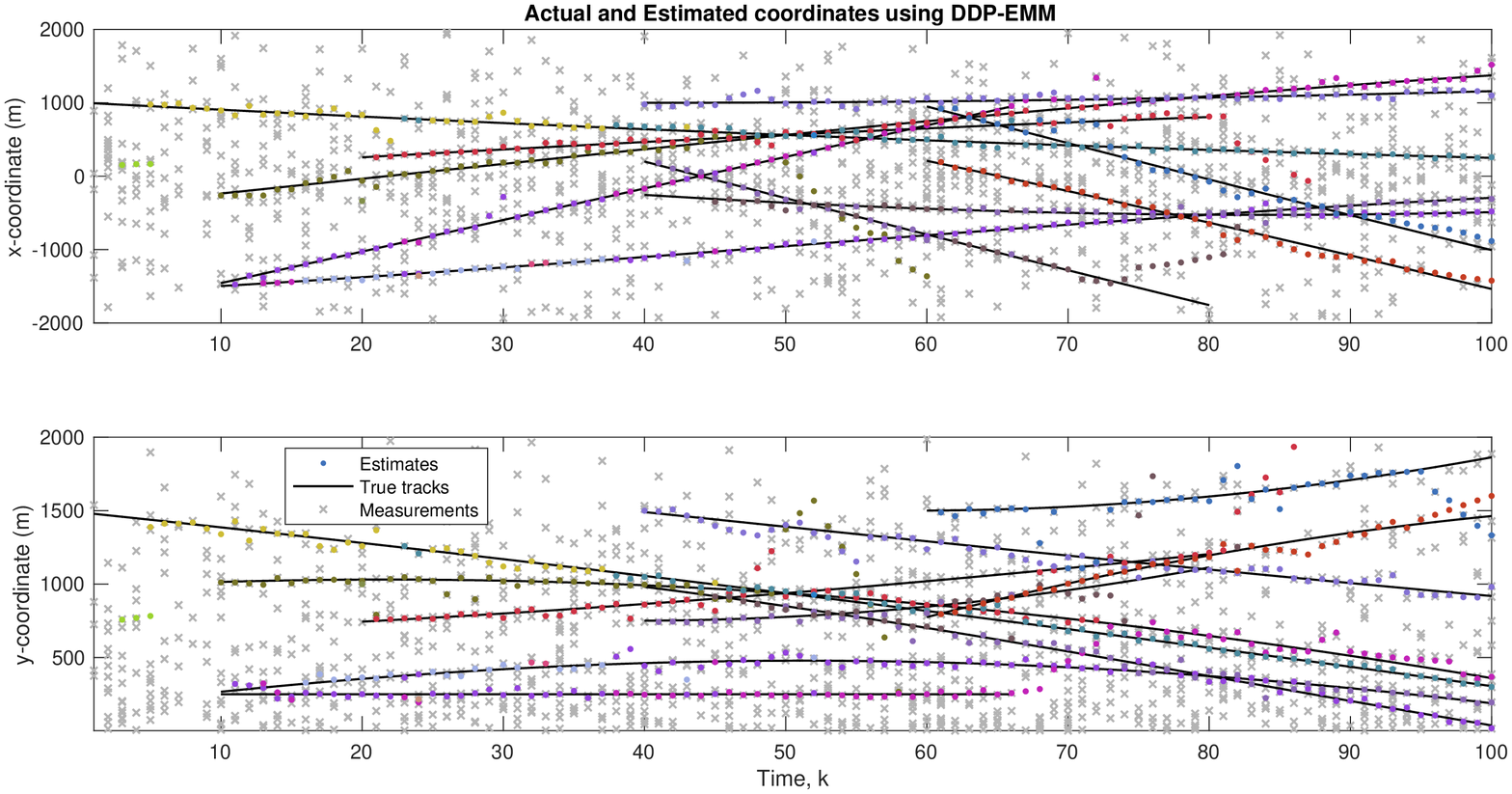}}}
\vspace*{-1mm}
\caption{  (a) Actual and estimated x and y coordinates through DPY-STP (b) Actual and estimated x and y coordinates through DDP-EMM.}
\vspace*{-.1in}
\end{figure*}

Figures \ref{xypy} and \ref{xydp} displays the actual and estimated coordinates through DPY-STP and DDM-EMM, respectively. We show the location estimation of objects through DPY-STP and DDP-EMM in Figures \ref{locpy} and \ref{locdp}, respectively. The Figure \ref{locpy} shows that DPY-STP has higher accuracy compared to DDP-EMM model. We can also demonstrate this using the OSPA metric with cut-off $c=100$ and order $p=1$. We observe that DPY-STP has a better performance compared to DDP-EMM as depicted in \ref{pydp}.

\begin{figure*}[htb]
\centering
\subfloat[ ]{ \label{locpy}
 \resizebox{3in}{!}{\includegraphics[width = 5cm, height = 3.5cm]{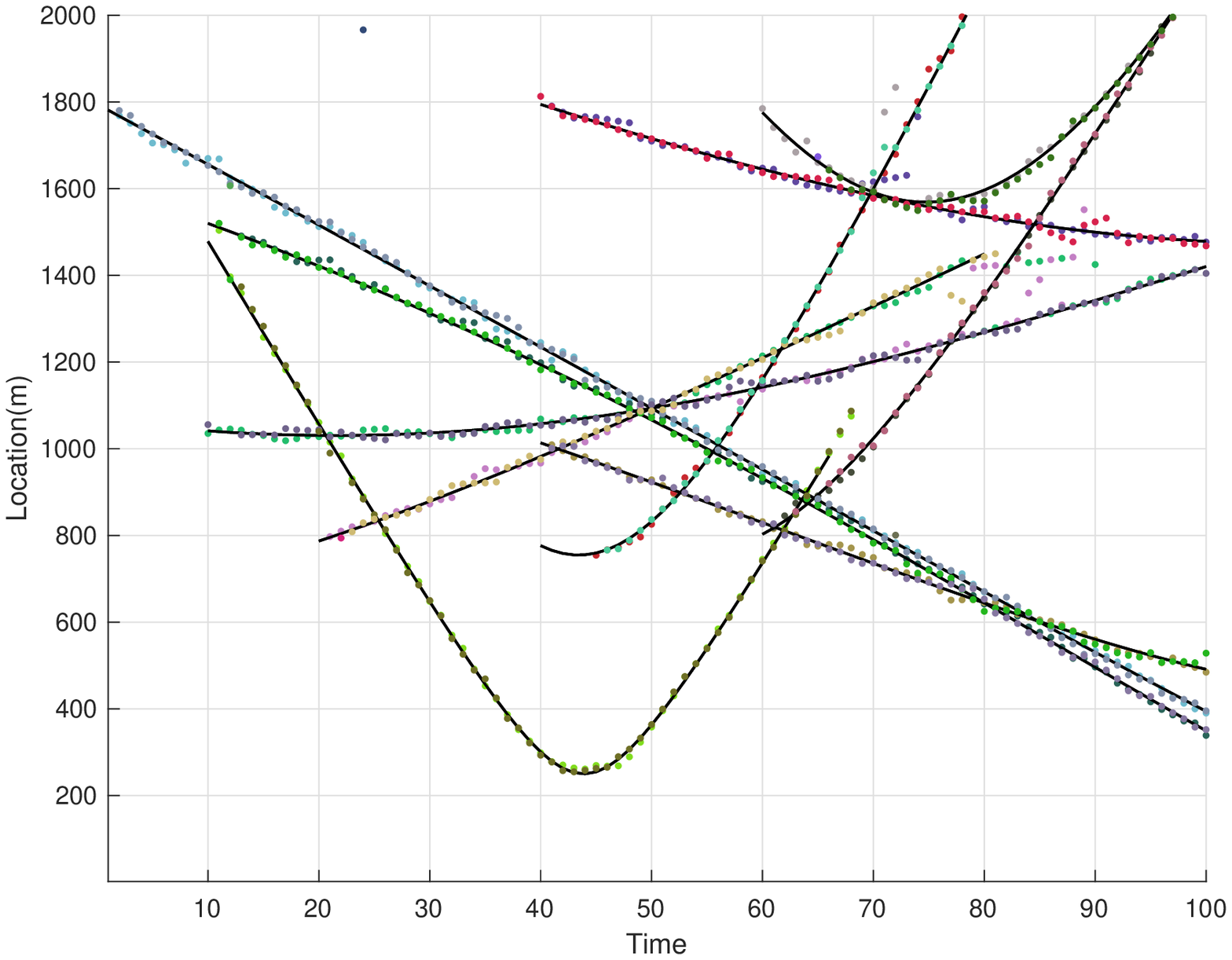}}}
\hspace*{-.1in}
\subfloat[ ]{ \label{locdp}
 \resizebox{3.5in}{!}{\includegraphics[width = 5cm, height = 3cm]{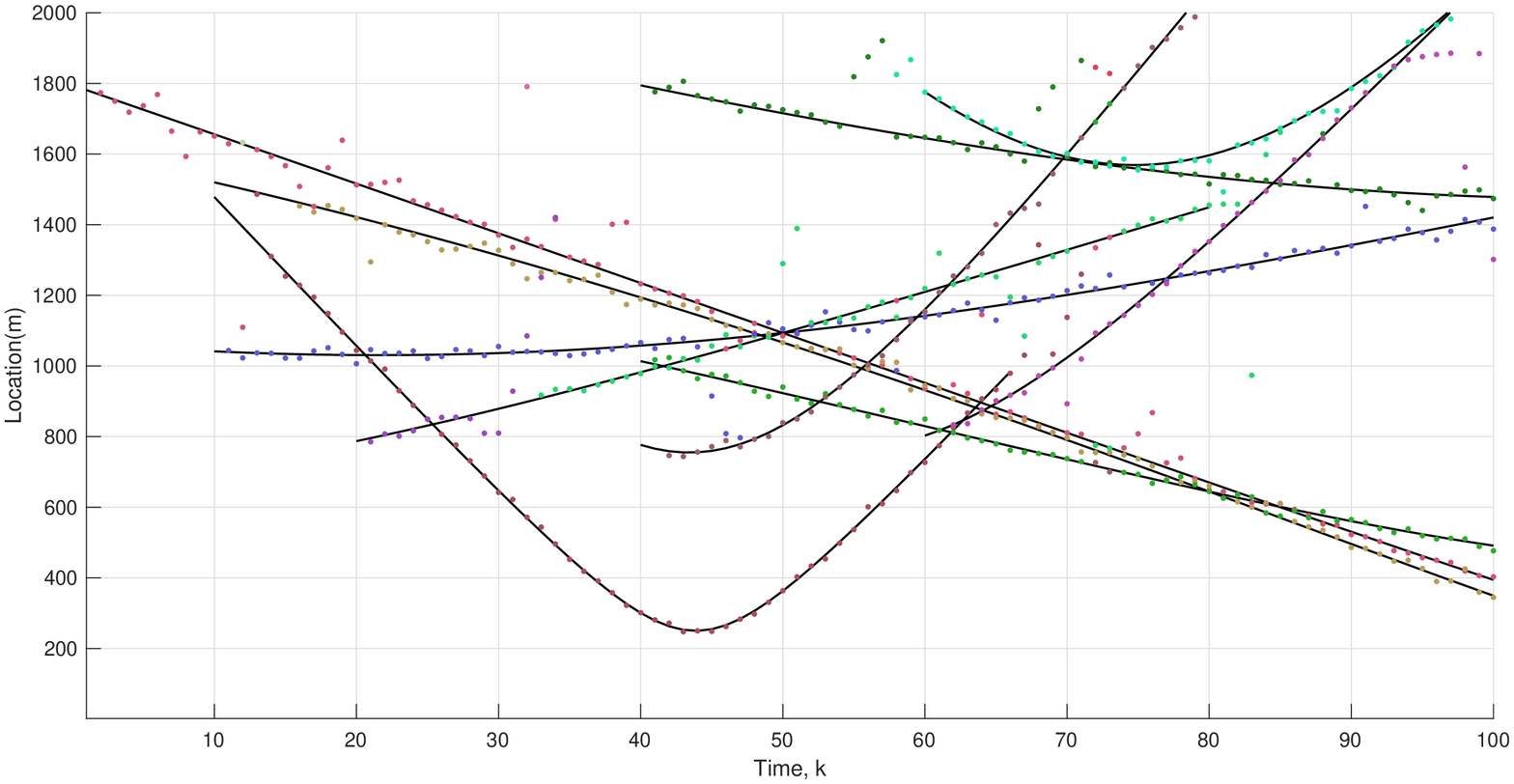}}}
\vspace*{-1mm}
\caption{  (a) Actual and estimated location through DPY-STP (b) Actual and estimated location through DDP-EMM.}
\vspace*{-.1in}
\end{figure*}

\begin{figure*}[!tbp]
 \begin{center}
 \resizebox{5in}{!}{\includegraphics{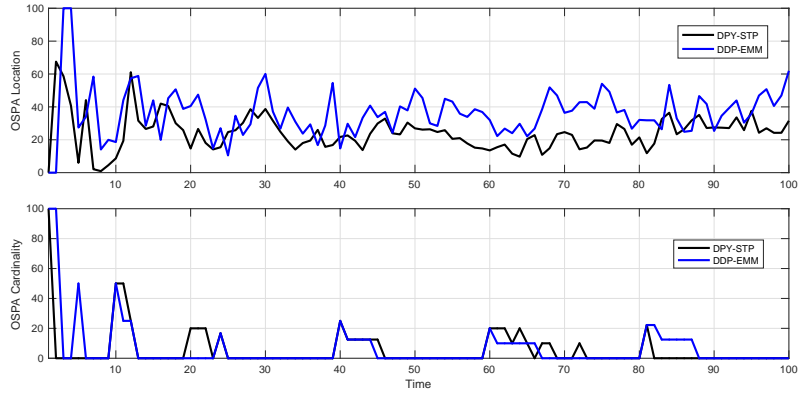}}
 \end{center}
\caption{OSPA comparison between DPY-STP (black) and DDP-EMM (blue) for cut-off $c = 100$ and order $p = 1$.}
 \label{pydp}
\end{figure*}

\section{Conclusion}

In this paper, we presented novel families of nonparametric processes that naturally captures the computational and inferential needs of a multi object tracking problem. We exploited dependent Dirichlet process and Pitman-Yor processes to model the objects and therefore tracking object trajectories. We showed that DDP-EMM is marginally a DP and DPY-STP is marginally a PY process and they follow the EPPF formula. We also derived the Gibbs sampler for both DDP-EMM and DPY-STP methods and manifested that the proposed Bayesian nonparametric framework can efficiently track the labels, cardinality, and object trajectories. Furthermore, MCMC implementation of the proposed tracking algorithms successfully verifies the simplicity and accuracy of these algorithms.

\appendixpage
\section{Appendix A}
\subsection{Proof of \ref{Gibbs}:}
\label{proof1}
\begin{proof}
The proof of \ref{Gibbs} follows the standard Bayesian nonparametric methods. We know that the base measure in DP($\alpha$, H) is the mean of the Dirichlet prior. The following lemma generalizes this fact.
\begin{lemma}(Ferguson 1973, \cite{ferg1973})
\label{mean}
If $G \sim \text{DP}(\alpha, H)$ and $f$ is any measurable function, then 
\begin{equation}
\mathbb{E}\Big[ \int f(\theta) dG(\theta)\Big] = \int f(\theta) dH(\theta) \notag
\end{equation}
\end{lemma}
Suppose that $A$ and $B$ are measurable sets. 
\begin{flalign}
P(\theta_{\ell,k}\in A , \bz_{\ell, k} \in B| \theta_{-\ell,k}, \bz_{-\ell, k}) = &\mathbb{E} \big[ \mathbbm{1}_{\theta_{\ell,k}}(A)\mathbbm{1}_{\bz_{\ell, k}}(B)| \theta_{-\ell,k}, \bz_{-\ell, k}\big]\label{d1}\\
=&\mathbb{E}\Big[\mathbb{E} \big[ \mathbbm{1}_{\theta_{\ell,k}}(A)\mathbbm{1}_{\bz_{\ell, k}}(B)| G, \theta_{-\ell,k}, \bz_{-\ell, k}\big] | \theta_{-\ell,k}, \bz_{-\ell, k}\Big]\label{d2}\\
= &\mathbb{E}\Big[ \int \mathbbm{1}_{\theta_{\ell,k}}(A)\mathbbm{1}_{\bz_{\ell, k}}(B) p(\bz_{\ell,k}|\theta_{\ell,k}, \bx_{\ell,k})d\bz_{\ell, k} dG({\theta_{\ell,k}}|{\theta_{-\ell,k}})\Big]\label{d3}
\end{flalign}
where \ref{d1} follows the definition of expected value, \ref{d2} is due to the law of iterated expectations, and $G(\theta)$ in \ref{d3} is the posterior dependent Dirichlet process given in \ref{posttheta}. Using lemma \ref{mean} 
\begin{flalign}
&\mathbb{E}\Big[ \int \mathbbm{1}_{\theta_{\ell,k}}(A)\mathbbm{1}_{\bz_{\ell, k}}(B) p(\bz_{\ell,k}|\theta_{\ell,k}, \bx_{\ell,k})d\bz_{\ell, k} dG({\theta_{\ell,k}}|{\theta_{-\ell,k}})\Big] =\\
 & \int \mathbbm{1}_{\theta_{\ell,k}}(A)\mathbbm{1}_{\bz_{\ell, k}}(B) p(\bz_{\ell,k}|\theta_{\ell,k}, \bx_{\ell,k})d\bz_{\ell, k} d\Big(\sum\limits_{\Theta_k - \{\theta_{\ell,k}\}} \Pi_1 \delta_{\theta}(\theta_{\ell,k}) + \sum\limits_{\substack{\theta \in \Theta^\star_{k|k-1}\setminus \Theta\\ \theta\neq \theta_{\ell,k}}}\Pi_2 \nu({\bf{\theta}^\star_{\ell, k-1}},{\bf{\theta_{\ell, k}}}) \delta_{\theta}({\bf{\theta_{\ell,k}}}) + \Pi_3 H(\theta_{\ell,k})\Big)\notag.
 \end{flalign}
 Using the Bayes rule we have:
 \begin{flalign}
 P(\theta_{\ell,k}\in A | \theta_{-\ell,k}, \mathcal{Z}_k) = \frac{ \int_{B} P(\theta_{\ell,k}\in A , \bz_{\ell, k} | \theta_{-\ell,k}, \bz_{-\ell, k})d\bz_{\ell, k}}{\int_{\Omega}P(\theta_{\ell,k}\in A , \bz_{\ell, k} | \theta_{-\ell,k}, \bz_{-\ell, k})d\bz_{\ell, k}}
 \end{flalign}
 and this concludes the claim in \ref{Gibbs}.
\end{proof}

 \subsection{Proof of Theorem \ref{thm2}}
 \label{proofthm2}
 \begin{proof}
To prove this theorem we check the conditions in the following theorem:
 \begin{pos}[Theorem 1, Tierney 1994 \cite{tierney1994}]
 \label{tierney}
Assume \text{K} is a $\pi$-irreducible and aperiodic Markov transition kernel such that $\pi \text{K} = \pi$. Then \text{K} is positive recurrent and $\pi$ is the unique invariant distribution of \text{K} and for almost all $x$ we have:
\begin{equation}
|| K^n(x,\cdot) - \pi ||_{TV} \longrightarrow 0
\end{equation}
where $||\cdot||_{TV}$ is the total variation norm. 
\end{pos}
The proof of invariance of the  posterior distribution for the Markov chain defined in \ref{Gibbs} is very similar to the proof of theorem 2 [escober 1994]. We only need to prove the aperiodicity and irreducibility of the Markov transition kernel with respect to the posterior distribution. \\
{\textit{Irreducibility:}} Assume that $B^k_{\theta} =\cup B^k_{j,\theta}$ is a partition where the elements of this partition, $B^k_{j,\theta}$, are the parameters configuration vector at time $k$ and $\pi_{j,k}(B^k_{j,\theta})$ is the probability measure associated for a fixed configuration. Note that the distribution $\pi_k$ at time $k$ has a unique distribution $\pi_k = \sum \pi_{j,k}(B^k_{j,\theta})$. Conditioning on a fixed configuration with $\pi_k(B^k_{j,\theta}) > 0$, both posterior and predictive distributions depends on distributions where posterior and $\pi_k$ take to be mutually absolutely continuous with the transition kernel $\text{K}(\theta_0, B^k_{j,\theta}) > 0$. The construction of transition kernel implies that for any $\theta_0$ the transition kernel is positive, $\text{K}(\theta_0, B^k_{j,\theta}) > 0$, therefore, $\text{K}(\theta_0, B^k_{\theta}) > 0$ with respect to $\pi_k$. Note that the posterior and $\pi_k$ are mutually absolutely continuous hence one can conclude that $\text{K}(\theta_0, B^k_{\theta}) > 0$ with respect to the posterior.\\
{\textit{Aperiodicity:}} Note that for $B^k_{\theta}$, we have $\pi_k(B^k_{\theta}) > 0$ which directly implies the aperiodicity of the kernel.
Therefore, the defined Markov chain sampler is irreducible, aperiodic, and invariant with respect to the posterior, hence, it satisfies the conditions in postulate \ref{tierney}.
\end{proof}

%\bibliographystyle{unsrt}  
%\begin{thebibliography}{1}
%\end{thebibliography}
\bibliography{Main-BNP-MOT-BM} 
\bibliographystyle{ieeetr}

\end{document}